\documentclass[11pt]{article}
\usepackage[utf8]{inputenc}
\usepackage{hyperref}

\usepackage{natbib}
\usepackage{colortbl}
\usepackage{ulem}
\usepackage{booktabs} 
\usepackage{amsmath}
\usepackage{amsthm}
\usepackage{amssymb}
\usepackage{fullpage}
\newtheorem{theorem}{Theorem}
\newtheorem{lemma}{Lemma}
\newtheorem{proposition}{Proposition}

\theoremstyle{definition}
\newtheorem{remark}{Remark}

\newcommand{\bet}{\beta_{0}}

\newcommand{\cD}{\mathcal{D}}

\newcommand{\prob}{\Pr}
\newcommand{\linner}{\left\langle}
\newcommand{\rinner}{\right\rangle}
\DeclareMathOperator*{\argmin}{arg\,min}

\DeclareMathOperator*{\E}{\mathbf{E}}
\newcommand{\re}{\mathbb{R}}
\renewcommand{\epsilon}{\varepsilon}

\newcommand{\inote}[1]{\textcolor{red}{(Shinji: #1)}}
\newcommand{\hnote}[1]{\textcolor{red}{(Honda: #1)}}
\newcommand{\tnote}[1]{\textcolor{red}{(Tsuchiya: #1)}}
\renewcommand{\inote}[1]{}
\renewcommand{\hnote}[1]{}
\renewcommand{\tnote}[1]{}

\newcommand{\simplex}[1]{\mathcal{P}_{#1}}

\title{
  Adversarially Robust Multi-Armed Bandit Algorithm \\
  with Variance-Dependent Regret Bounds
}
\usepackage{times}

\author{
  Shinji Ito\footnote{
    NEC Corporation and RIKEN AIP; \texttt{i-shinji@nec.com}.
    Supported by JST, ACT-I Grant Number JPMJPR18U5, Japan.
  }
  \and
  Taira Tsuchiya\footnote{
    Kyoto University and RIKEN AIP; 
    \texttt{tsuchiya@sys.i.kyoto-u.ac.jp}.
    Supported by JST, ACT-X Grant Number JPMJAX210E,
    and JSPS, KAKENHI Grant Number JP21J21272.
  }
  \and
  Junya Honda\footnote{
    Kyoto University and RIKEN AIP; \texttt{honda@i.kyoto-u.ac.jp}.
    Supported by JSPS, KAKENHI Grant Number 21K11747, Japan.
  }
}

\begin{document}

\maketitle

\begin{abstract}%
  This paper considers the multi-armed bandit (MAB) problem and provides a new best-of-both-worlds (BOBW) algorithm that works nearly optimally in both stochastic and adversarial settings. In stochastic settings, some existing BOBW algorithms achieve tight gap-dependent regret bounds of $O(\sum_{i: \Delta_i>0} \frac{\log T}{\Delta_i})$ for suboptimality gap $\Delta_i$ of arm $i$ and time horizon $T$. As \citet{audibert2007tuning} have shown, however, that the performance can be improved in stochastic environments with low-variance arms. In fact, they have provided a stochastic MAB algorithm with gap-variance-dependent regret bounds of $O(\sum_{i: \Delta_i>0} (\frac{\sigma_i^2}{\Delta_i} + 1) \log T )$ for loss variance $\sigma_i^2$ of arm $i$. In this paper, we propose the first BOBW algorithm with gap-variance-dependent bounds, showing that the variance information can be used even in the possibly adversarial environment. Further, the leading constant factor in our gap-variance dependent bound is only (almost) twice the value for the lower bound. Additionally, the proposed algorithm enjoys multiple data-dependent regret bounds in adversarial settings and works well in stochastic settings with adversarial corruptions. The proposed algorithm is based on the follow-the-regularized-leader method and employs adaptive learning rates that depend on the empirical prediction error of the loss, which leads to gap-variance-dependent regret bounds reflecting the variance of the arms.
\end{abstract}


\section{Introduction}
In this paper,
we consider the multi-armed bandit (MAB)
problem where
a player is given $K$ arms.
In each round $t \in [T]=\{1, 2, \ldots, T\}$ for time horizon $T$,
the environment determines the loss vector $\ell(t) = ( \ell_1(t), \ell_2(t), \ldots, \ell_K(t) )^\top \in [0, 1]^K$
and then the player chooses an arm $I(t) \in [K] := \{ 1, 2, \ldots, K \}$ without knowing $\ell(t)$.
After that,
the player can observe only the loss $\ell_{I(t)}(t)$ for the chosen arm.
Player performance is measured by means of the regret $R(T)$ defined as
\begin{align}
  \label{eq:defRT}
  R_{i^*}(T) = \E \left[
    \sum_{t=1}^T \ell_{I(t)}(t)
    -
    \sum_{t=1}^T \ell_{i^*}(t)
  \right],
  \quad
  R(T) = \max_{i^* \in [K]}R_{i^*}(T) ,
\end{align}
where the expectation is taken with respect to the algorithm's internal randomness and the randomness of loss vectors $\{ \ell(t) \}_{t=1}^{T}$.
There are mainly two types of environments to determine the loss vector $\ell(t)$, that is,
the stochastic MAB \citep{lai1985asymptotically,auer2002finite} and the adversarial MAB \citep{auer2002nonstochastic},
which are detailed as follows.


In stochastic MAB,
we assume that there exists a fixed distribution $\cD$ over $[0, 1]^K$,
from which the loss vectors $\ell(t)$ for $t\in [T]$ are i.i.d.
It is known that the difficulty of the stochastic MAB can be characterized by the
\textit{suboptimality gap} $\Delta_i = \mu_i - \mu_{i^*}$,
where $\mu_i = \E_{\ell \sim \cD } [ \ell_i ]$ represents the expected loss for the arm $i$ and
$i^* \in \argmin_{i \in [K]} \{ \mu_i \}$ represents an optimal arm.
The optimal regret bound given the suboptimality gap can be expressed as $R(T) = \allowbreak O(\sum_{i: \Delta_i > 0} \frac{\log T}{\Delta_i})$,
which has been achieved by several algorithms,
including the UCB1 policy \citep{auer2002finite}.
In addition, there have been many studies to incorporate richer information of the loss distribution
not limited to the expected loss.
For example,
the UCB-V algorithm by
\citet{audibert2007tuning} achieves a variance-dependent regret bound of 
$R(T) \leq 10 \sum_{i: \Delta_i > 0} (\frac{\sigma_i^2}{\Delta_i} + 2) \log T $,
where $\sigma_i^2=\E_{\ell \sim \cD} [ (\ell_i - \mu_i)^2 ]$ is the loss variance of the $i$-th arm.
The DMED algorithm \citep{honda_dmed} is then proposed to achieve a bound that is
optimal for general loss distributions over $[0,1]^K$,
and thus is optimal in terms of $\Delta_i$ and $\sigma_i^2$.
Recently, Thompson sampling \citep{thompson_kaufmann, thompson_further} known for its practicality
has also been extended so as to have the optimality for general loss distributions
\citep{riou_nonparametric, baudry_nonparametric}.

In adversarial MAB,
no generative model for the losses is assumed,
and the environment may choose loss vectors $\ell(t)$ depending on the algorithm's choices $\{ I(s) \}_{s=1}^{t-1}$ so far,
in an adversarial manner.
For this problem,
algorithms based on follow-the-regularized-leader methods with Tsallis entropy regularization achieve $O(\sqrt{KT})$-regret bounds
\citep{audibert2009minimax,zimmert2019connections},
which matches the lower bound of $\Omega(\sqrt{KT})$ shown, e.g., by \citet{auer2002nonstochastic}.

In general,
stochastic MAB algorithms including UCB1 and Thompson sampling
work poorly
for adversarial MAB
while most adversarial MAB algorithms show suboptimal performance for stochastic MAB.
This means that,
in practice,
it is important to choose an algorithm
appropriate for the environment,
but it is difficult to know in advance how the environment will behave.

A promising approach to this issue is to develop
best-of-both-worlds (BOBW) algorithms \citep{bubeck2012best} that work well for both stochastic and adversarial environments.
One representative BOBW algorithm is the Tsallis-INF algorithm proposed by \citet{zimmert2021tsallis},
which achieves regret bounds of $O(\sum_{i: \Delta_i > 0} \frac{\log T}{\Delta_i} )$ for stochastic MAB and
$O(\sqrt{KT})$ for adversarial MAB.
The Tsallis-INF algorithm also works well in
the \textit{stochastic environment with adversarial corruption},
an intermediate setting between stochastic and adversarial environments.
In this setting,
loss vectors are generated stochastically and then corrupted by the adversary.
The \textit{corruption level} parameter $C > 0$ quantifies the total amount of corruption.
The Tsallis-INF algorithm achieves an
$ O( \sum_{i: \Delta_i > 0} \frac{\log T}{\Delta_i} + \sqrt{\sum_{i: \Delta_i > 0} \frac{C \log T}{\Delta_i}} ) $-regret bound,
which have been improved to a more refined expression
by \citet{masoudian2021improved}.

One challenge for BOBW algorithms is how to exploit information of the loss distribution richer than the expected loss
for the stochastic environment.
In the stochastic MAB algorithms, the loss variance and the loss distribution can be naturally incorporated to the regret bound
by algorithms using the empirical variance \citep{audibert2007tuning} and the empirical distribution \citep{honda_dmed}, respectively.
On the other hand for the BOBW framework, we need to carefully explore and estimate even the expected loss to guarantee the adversarial robustness,
which makes it much more difficult to construct an algorithm adaptive to the variance or the loss distribution.
Unlike UCB-type BOBW algorithms
\citep{bubeck2012best,seldin2017improved,auer2016algorithm},
the Tsallis-INF algorithm is based on totally different frameworks and it seems challenging to
simply combine it with stochastic MAB algorithms such as the UCB-V algorithm.



\subsection{Contribution of This Work}
The main contribution of this work is 
the proposal of an algorithm with a variance-dependent regret bound for the stochastic environment,
while preserving the BOBW property and the robustness to adversarial corruption.
This result can be viewed as the first step for the construction of BOBW algorithms that can
exploit information of the loss distribution not limited to the expected loss.

The proposed algorithm is based on the optimistic follow-the-regularized-leader (OFTRL) \citep{rakhlin2013online,rakhlin2013optimization} approach.
In this approach,
we use optimistic predictions to
reduce the variances in unbiased estimators for loss vectors.
These are also referred to as hint vectors.
This approach is employed in such existing MAB algorithms as
ones by \citet{wei2018more,ito2021parameter},
as well.
The major difference between such existing algorithms and our proposed algorithm
is in the choice of regularizer functions,
learning rates,
and optimistic predictions.
More specifically,
the proposed algorithm uses a learning rate dependent on
the empirical prediction error of the hint vectors, which is naturally linked to the loss variance in the stochastic environment.
This modification plays
a central role in obtaining a variance-dependent regret bound.
The regularizer functions of the proposed algorithm are composed of a combination of
(1) entropy terms for the complement $(1-p_i)$ of the probability $p_i$ of choosing the arm $i$
and
(2) logarithmic barriers that are used in existing works by \citet{wei2018more,ito2021parameter}.
The former entropy component has the role of removing the impact of the variance of the optimal arm from the regret bound.
We note that the idea of adding the entropy for $(1-p_i)$ has already been introduced by \citet{zimmert2019beating}.
However,
their motivation for adopting this idea was to address combinatorial semi-bandits,
and our study is the first to show that this idea is also useful in deriving variance-dependent regret bounds.
A recent study on combinatorial semi-bandit by \citet{ito2021hybrid} has also employed a similar regularizer.

The performance of the proposed algorithm can be summarized as follows:
\begin{theorem}
  \label{thm:main}
  In any stochastic setting with a unique optimal arm,
  the proposed algorithm achieves
  \begin{align}
    \label{eq:thm1-bound-sto}
    R(T)
    \leq
    \left(
    2
    (1+\epsilon)
    +
    \sum_{i \in [K] \setminus \{ i^* \}} \max\left\{4\frac{\sigma_i^2}{\Delta_i}
    +
    c
    \log \left(
      1 + \frac{\sigma_i^2}{\Delta_i}
    \right), 2(1+\epsilon)\right\}
    \right)
    \log T
%
    +
    o (\log T),
  \end{align}
  where $i^* \in [K]$ is the optimal arm,
  $\epsilon \in (0, 1/2]$ is an input parameter for the algorithm,
  and $c = O((\log \frac{1}{\epsilon})^2)$.
  Further,
  in adversarial settings,
  the regret for the proposed algorithm satisfies
  \begin{align}
    \label{eq:thm1-bound-adv}
    R(T) \leq \sqrt{K \log T \min\left\{ T, 4L^*,  4 (T-L^*),  4 Q_{\infty}  \right\}} + O(K \log T)
  \end{align}
  for 
  $L^* = \min_{i^* \in [K] }\E [ \sum_{t=1}^T \ell_{i^*} (t) ] $ and 
  $Q_{\infty} = \min_{\bar{\ell} \in [0,1]^K} \E [ \sum_{t=1}^T \|\ell_i(t) - \bar{\ell}\|_{\infty}^2 ]$.
  In addition,
  in corrupted stochastic settings,
  we have $R(T) \leq \mathcal{R} + O(\sqrt{C \mathcal{R}})$,
  where $\mathcal{R}$ is the RHS of \eqref{eq:thm1-bound-sto}.
\end{theorem}
Note that the proposed algorithm does not require any prior knowledge on
$\sigma^2_i$, $\Delta_i$, $L^*$, $Q_\infty$,
and $C$.

\begin{table*}%
\caption{
  List of parameters.
}
  \label{table:parameters}
    \centering
    \begin{tabular}{lll}
        \toprule
        Parameter & Region &  Description
        \\
        \midrule
        $\ell(t) = ( \ell_1(t), \ldots, \ell_K(t) )^\top $&$  [0, 1]^K$& Loss vector for the $t$-th round
        \\
        $\mu_i = \E[ \ell_i(t) ]$&$  [0, 1]$& Expected loss for the arm $i$ (sto.~setting)
        \\
        $\Delta_i = \mu_i - \min_{i^* \in [K]} \mu_{i^*}$&$  [0, 1]$& Suboptimality gap for the arm $i$ (sto.~setting)
        \\
        $\sigma_i^2 = \E[ (\ell_i(t) - \mu_i )^2 ] $&$ [0, 1/4]$& Variance for the arm $i$ (sto.~setting)
        \\
        $L^* = \min_{i^* \in [K]} \E[ \sum_{t=1}^T \ell_{i^*}(t) ]$&$ [0, T]$& Cumulative loss for the optimal arm
        \\
        $Q_{\infty} = \min_{\bar{\ell} \in \re^K} \E [ \sum_{t=1}^T \| \ell(t) - \bar{\ell} \|_{\infty}^2 ] $&$ [0, T/4]$& Empirical variation of loss vectors 
        \\
        $V_1 = \sum_{t=1}^{T-1} \| \ell(t) - \ell(t+1) \|_1$&$ [0, T]$ & Path-length of loss vectors  
        \\
        $C = \sum_{t=1}^T \| \ell(t) - \ell'(t) \|_{\infty}$&$  [0, T]$ & Corruption level (sto.~setting with adv.~corruption), 
        \\
        &&
        where $\ell'(t)$ follows a fixed distribution
        \\
        \bottomrule
    \end{tabular}
\end{table*}%
\begin{table*}[ht]%
\caption{
  Regret bounds for multi-armed bandit.
}
  \label{table:regretbound}
    \centering
    \begin{tabular}{llccccccccccc}
        \toprule
        Environment & Bound & UCB-V & Tsallis-INF  & LB-INF & \textbf{[Proposed]}
        \\
        \midrule
        Stochastic & $\Delta$-dependent & \checkmark & \checkmark & \checkmark & \checkmark
        \\
        & $(\Delta, \sigma^2)$-dependent & \checkmark &  & & \checkmark
        \\
        \midrule
        Adversarial & Worst case: $\tilde{O}(\sqrt{KT})$ & & \checkmark & \checkmark & \checkmark
        \\ 
        & First order: $\tilde{O}(\sqrt{K L^*})$ & & & \checkmark & \checkmark
        \\ 
        & Second order: $\tilde{O}(\sqrt{K Q_{\infty}})$ & & & \checkmark & \checkmark
        \\
        \midrule
        Sto.~with adv. & $(\Delta, C)$-dependent & & \checkmark & \checkmark & \checkmark        \\
        corruption & $(\Delta, \sigma^2, C)$-dependent & &  & & \checkmark        \\
        \bottomrule
    \end{tabular}
\end{table*}%

Qualitative features of the proposed algorithm and a comparison of it with existing algorithms are summarized in Table~\ref{table:regretbound}.
The proposed algorithm is superior to
the UCB-V algorithm by \citet{audibert2007tuning},
in having regret bounds for adversarial settings and for corrupted stochastic settings.
In addition,
the proposed algorithm is superior to
such existing BOBW algorithm as the Tsallis-INF algorithm by \citet{zimmert2021tsallis} and the LB-INF algorithm \citet{ito2021parameter},\footnote{
  In this paper,
  we refer to Algorithm 1 in the paper \citep{ito2021parameter} as LB-INF,
  which stands for Logarithmic Barrier-Implicit Normalized Forecaster~\citep{audibert2009minimax}.
}
in that it enjoys variance-dependent regret bounds in (corrupted) stochastic settings.
To our knowledge,
the proposed algorithm is the first to have,
simultaneously,
variance-dependent regret bounds in stochastic settings and a BOBW regret guarantee.

Table~\ref{table:regretbound-sto} summarizes asymptotic regret bounds in stochastic settings.
To quantitatively compare the proposed algorithm with asymptotically optimal
algorithms for the stochastic environment such as the DMED algorithm,
we show that the best-possible bound in terms of $(\Delta_i,\sigma_i)$ is roughly expressed as
$\sum_{i \neq i^*} ( 2 \frac{\sigma^2}{\Delta_i} + 0.5 \log (1 +\frac{\sigma_i^2}{\Delta_i})+1)$.
The proposed algorithm has input parameter $\epsilon \in (0, 1/2]$
with a regret bound of
$\sum_{i \neq i^*} \max\{ 4 \frac{\sigma^2}{\Delta_i} + 4.2 \log (1 +\frac{\sigma_i^2}{\Delta_i}),\, 2.4 K\}$ for $\epsilon = 0.2$,
as long as the optimal arm is unique.
This regret bound is close to twice the optimal algorithms, and at most $2.5$ times in the worst case.
In particular, this bound is much better than the UCB-V algorithm, which is directly designed for a variance-dependent bound.
In Table~\ref{table:regretbound-sto},
(IW) and (RV) stand,
respectively,
for the importance-weighted sampling estimator and reduced-variance estimator
used in the Tsallis-INF algorithm by \citet{zimmert2019connections}.
As we have $\sigma_i^2 \leq 1/4$ from the assumption of $\ell(t) \in [0,1]^K$,
the proposed algorithm has the bound of
$\sum_{i \neq i^*}\left( \frac{1}{\Delta_i} + 4.2 \log (1 + \frac{1}{4\Delta_i}) +2.4\right)$,
which in most cases is better than the bounds with Tsallis-INF.

\begin{table*}%
\caption{
  Regret bounds for the stochastic multi-armed bandit.
}
  \label{table:regretbound-sto}
    \centering
    \begin{tabular}{llccccccccccc}
        \toprule
         & $(\Delta_i,\sigma_i^2)$-dependent bound on $\lim_{T \to \infty} R(T)/\log T $& 
        \\
	\midrule
        Lower bound& $\approx \sum_{i\neq i^*} \left(2 \frac{\sigma_i^2}{\Delta_i} + \frac{1}{2} \log\left( 1 + \frac{\sigma_i^2}{\Delta_i} \right)+1\right) $ (with relative err. $\le 6\%$)
        \\
          \midrule
        UCB-V & $\sum_{i: \Delta_i > 0} \left(10 \frac{\sigma_i^2}{\Delta_i} + 20 \right)$ & 
        \\
        DMED etc. & (optimal, in a form without an explicit dependence on $(\Delta_i,\sigma_i^2)$)& 
        \\
        Tsallis-INF (IW) &
        $  \sum_{i \neq i^*} \frac{4}{\Delta_i}+4 $ 
        &
        \\
        Tsallis-INF (RV) &  $\sum_{i \neq i^*} \frac{1}{\Delta_i} + 28K  $ &
        \\
	\textbf{[Proposed]} & $\sum_{i\neq i^*} \max\left\{4 \frac{\sigma_i^2}{\Delta_i} + c(\epsilon) \log\left( 1 + \frac{\sigma_i^2}{\Delta_i}\right),\,2(1+\epsilon)\right\}+2(1+\epsilon) $ 
  & 
	\\
  & $
  \left(
    \leq
    \sum_{i\neq i^*} \max\left\{4 \frac{\sigma_i^2}{\Delta_i} + 4.2 \log\left( 1 + \frac{\sigma_i^2}{\Delta_i}\right),\,2.4\right\}+2.4
    \quad
    \mbox{if }
    \epsilon := 0.2
  \right)
  $
  \\
        \bottomrule
    \end{tabular}
\end{table*}%
\begin{table*}[ht]%
\caption{
  Regret bounds for the adversarial multi-armed bandit.
}
  \label{table:regretbound-adv}
    \centering
    \begin{tabular}{llccccccccccc}
        \toprule
         & Bound on $R(T)$& 
        \\
  \midrule
  Lower bound & $\frac{1}{20} \min\{ \sqrt{KT}, T \}$
  \\
        \midrule
        Tsallis-INF (IW) &  $4\sqrt{KT} + 1$ & 
        \\
        Tsallis-INF (RV) &  $2 \sqrt{KT} + O(K \log T)$& 
        \\
        Tsallis-INF (RV-const) &  $ \sqrt{2KT} + 48 K$& 
        \\
        LB-INF &  
        $3\sqrt{K \log T \min\{ T, 4 L^*, 4 (T - L^*), 4 Q_{\infty}, 32 V_1 \}} + O(K\log T)$  & 
        \\
	\textbf{[Proposed]} & 
        $\sqrt{K \log T \min\{T, 4 L^*, 4(T - L^*) , 4 Q_{\infty} \}} + O(K \log T)$  & 
	\\
        \bottomrule
    \end{tabular}
\end{table*}%

Comparisons of regret upper bounds for adversarial settings
are given
in Table~\ref{table:regretbound-adv}.
Tsallis-INF (RV-const) in this table refers to the algorithm in
\citet[Theorem~10]{zimmert2019connections}.
This algorithm achieves the smallest leading constant factor for $O(\sqrt{KT})$-regret upper bounds among those known,
whereas it does not have an $O(\log T)$-regret bound for stochastic settings
unlike the other algorithms in Table~\ref{table:regretbound-adv}.
Worst-case regret bounds for LB-INF by \citet{ito2021parameter} and the proposed algorithm are of $O(\sqrt{KT \log T})$,
which includes an additional factor of $O(\sqrt{\log T})$ that is not found in those of 
Tsallis-INF algorithms.
By way of contrast,
LB-INF and the proposed algorithm have regret bounds dependent on
the cumulative loss $L^* = \sum_{t=1}^T \ell_{i^*}(t)$ and
the total variation $Q_{\infty} = \sum_{t=1}^T \| \ell(t) - \bar{\ell} \|_{\infty}^2$,
which are the most fundamental quantities to measure the difficulties of problem instances.
These bounds,
called data-dependent bounds,
mean that the algorithms work more effectively for environments having certain characteristics,
such as situations in which losses have only small variances or the optimal arm has only a small cumulative loss.
\begin{remark}
  Unlike LB-INF,
  the proposed algorithm does not have a bound
  dependent on the path length $V_1 = \sum_{t=1}^{T-1} \| \ell(t) - \ell(t+1) \|_1$.
  However,
  a minor modification to the proposed algorithm can yield a path-length regret bound,
  in exchange for a larger constant factor in the regret bounds.
  More precisely,
  by modifying the algorithm in the way described in Appendix~\ref{sec:path-length},
  we can obtain a regret bound of
  \begin{align}
    \label{eq:bound-pathlength}
    R(T)
    \leq
    \sqrt{
      \frac{K}{1-2\eta}
      \left(
      \min\left\{
        T,
        4 L^*,
        4 (T - L^*),
        4 Q_{\infty},
        \frac{8}{\eta} V_1
      \right\}
      +
      \frac{K}{\eta}
      \right)
      \log T
    }
    +
    O(K \log T)
  \end{align}
  in adversarial settings,
  where $\eta \in (0, 1 / 2 )$ is an input parameter.
  In corrupted stochastic settings,
  this modified algoirthm achieves a regret bound of 
  $
  R(T) \leq  \frac{1}{1-2\eta} \mathcal{R} + 
  O \left( \sqrt{ \frac{1}{1-2\eta} C \mathcal{R} } + K \sqrt{\frac{ \log T }{\eta (1-2\eta)}} \right)
  $,
  where $\mathcal{R}$ is the RHS of \eqref{eq:thm1-bound-sto}.
  More details on the modification and regret bounds can be found in Appendix~\ref{sec:path-length}.
\end{remark}

To show regret bounds for stochastic settings (with adversarial corruption),
we use the self-bounding technique \citep{wei2018more,zimmert2021tsallis,gaillard2014second},
which uses a regret {\it lower} bound called a self-bounding constraint
\citep{zimmert2021tsallis}
to derive a regret upper bound.
Previous studies used this technique
to derive a regret bound for stochastic settings
by
a regret upper bound dependent on 
$\{ p(t) \}$
and the self-bounding constraint,
where $\{p(t)\}$ is the distribution of the chosen arm.
In contrast,
we present a regret upper bound dependent on $\{ p(t) \}$ and the squared prediction error of the loss,
which naturally leads to a variance-dependent regret bound.
To our knowledge,
the proposed algorithm is the first to achieve a variance-dependent regret bound
using the follow-the-regularized-leader approach and a self-bounding technique.


\subsection{Related Work}
Since \citet{bubeck2012best} pioneered the study on BOBW algorithms for the MAB problem,
the scope of their research has been extended to a variety of problem settings,
including the problem of prediction with expert advice \citep{amir2020prediction,mourtada2019optimality,de2014follow,gaillard2014second,luo2015achieving},
combinatorial semi-bandits \citep{zimmert2019beating,ito2021hybrid},
online linear optimization \citep{huang2016following},
linear bandits \citep{lee2021achieving},
online learning with feedback graphs \citep{erez2021best},
and episodic Markov decision processes \citep{jin2020simultaneously,jin2021best}.
Among these,
the studies particularly relevant to this study are those on follow-the-regularized-leader (FTRL)-based algorithms,
e.g.,
by \citet{zimmert2021tsallis,wei2018more,zimmert2019beating,gaillard2014second,erez2021best,jin2021best}.
For such algorithms,
regret bounds for the stochastic settings are shown via the so-called self-bounding technique \citep{zimmert2021tsallis,masoudian2021improved}.
One advantage of this technique is that it naturally leads, as well,
to regret bounds for stochastic settings with adversarial corruption.
Typically,
when we have a regret bound of $O(\mathcal{R})$ in a stochastic setting,
we get a bound of $O(\mathcal{R} + \sqrt{C\mathcal{R}})$ in a corrupted stochastic setting 
as shown,
e.g.,
in existing studies by \citet{zimmert2021tsallis,ito2021hybrid,ito2021optimal,ito2021parameter,jin2021best,erez2021best},
and this study.
For the MAB problem,
\citet{masoudian2021improved} have provided a more refined regret bound for corrupted stochastic settings 
by using a self-bounding technique.
It is not clear at this time,
however,
whether their approach can be applied to our proposed algorithm.


In the research area of adversarial MAB,
algorithms with various data-dependent regret bounds have been developed.
Typical examples of these are first-order bounds dependent on cumulative loss and
second-order bounds depending on sample variances in losses.
\citet{allenberg2006hannan} have provided an algorithm with a first-order regret bound of $O(\sqrt{K L^* \log K})$,
where $L^*$ is given in Table~\ref{table:parameters}.
Second-order regret bounds
have been shown in some studies,
e.g.,
by \citet{hazan2011better,wei2018more,bubeck2018sparsity}.
We note that the algorithm by \citet{wei2018more} has a bound dependent on the sample variance for the best arm,
which is better than those dependent on $Q_{\infty}$,
as achieved by the proposed algorithm in this paper.
Other examples of data-dependent bounds include
path-length bounds and sparsity-dependent bounds,
which have been shown,
e.g.,
by \citet{bubeck2019improved,bubeck2018sparsity,wei2018more}.


\section{Problem Setup}
\label{sec:setting}
In this section,
we formalize the environments of the problem considered in this paper.
In each round $t \in [T]$,
the environment chooses a loss vector $\ell(t) = (\ell_{1}(t) , \ell_{2}(t), \ldots, \ell_{K}(t))^\top \in [0, 1]^K$ and the player chooses an arm $I(t) \in [K]$.
The player then observes the incurred loss $\ell_{I(t)} (t)$ for the chosen arm.
For a sake of simplicity,
we assume that $T\ge \max\{55,K\}$ and is known to the player.\footnote{
  The assumption that $T$ is known in advance can be removed via
  the well-known doubling trick \citep{auer2010ucb,besson2018doubling},
  in exchange for increasing a constant factor in the regret bound. 
}
The performance of the player is evaluated in terms of the (expected) regret $R(T)$ defined
in \eqref{eq:defRT}.
We consider the following settings for the environment:

\paragraph{Stochastic environment}
In a stochastic setting,
$\ell(t) $ follows an unknown distribution $\cD$ over $[0, 1]^K$,
i.i.d.~for $t = 1, 2, \ldots, T$.
The expectation and the variance of $\ell_i(t)$ of each arm $i$ are denoted by
$\mu_i = \E_{\ell \sim \cD} [ \ell_i ] \in [0,1]$ and
$\sigma_i^2 = \E_{\ell \sim \cD} [ ( \ell_i - \mu_i )^2 ] \in [0,1/4]$, respectively.
We denote the suboptimality gap by
$\Delta_i = \mu_i - \mu_{i^*}$,
where $i^* \in \argmin_{i \in [K]} \mu_i$ represents the optimal arm.


\paragraph{Adversarial environment}
In the adversarial setting,
we do not assume any generative models for the loss vector $\ell(t)$,
but
it can be chosen in an adversarial manner.
More precisely,
in each round,
$\ell(t)$ may depend on the history of loss vectors and
chosen arms $\{ (\ell(s), I(s)) \}_{s=1}^{t-1}$.

\paragraph{Stochastic environment with adversarial corruption}
A stochastic setting with adversarial corruption,
or a corrupted stochastic setting,
is an intermediate setting between
a stochastic setting and an adversarial setting.
In such a setting,
a \textit{temporary loss} $\ell'(t) \in [0, 1]^K$ is generated from an unknown distribution $\cD$,
and the adversary then corrupts $\ell'(t)$ to determine $\ell(t)$.
We define the \textit{corruption level} $C \geq 0$ by
$
  C = \E \left[ \sum_{t=1}^T \left\|  \ell'(t) - \ell(t)  \right\|_{\infty} \right]
$.
Note that we have $C \leq T$
from this definition of $C$ and the assumption of $\ell(t), \ell'(t) \in [0, 1]^K$.
If $C = 0$,
this setting will be equivalent to the stochastic setting with the distribution $\cD$.
If there is no restriction on $C$,
the setting will be equivalent to the adversarial setting,
as $\ell(t)$ can be arbitrarily chosen.
For each $i \in [K]$,
we define $\mu_i$, $\sigma_i^2$, and $\Delta_i$ in terms of the distribution $\cD$ that generates $\ell'(t)$,
similarly to that done in stochastic settings without adversarial corruption.

In a corrupted stochastic setting,
regret is defined in terms of the loss $\ell(t)$ after corruption,
rather than in terms of the temporary loss $\ell'(t)$ before corruption,
as given in \eqref{eq:defRT}.
We note that some existing studies consider an alternative regret notion $R'(T)$ that is defined in terms of $\ell'(t)$.
A regret bound on $R(T)$ immediately leads to a bound on $R'(T)$.
In fact,
as $|R(T) - R'(T)| = O(C)$ follows from the definition,
a regret bound of $R(T) \leq \mathcal{R} + O(\sqrt{C\mathcal{R}})$ leads to
$R'(T) \leq \mathcal{R} + O(\sqrt{C \mathcal{R} } + C)
\leq O(\mathcal{R} + C)$,
where the second inequality follows from the AMGM inequality.


\section{Proposed Algorithm}
The proposed algorithm is in the \textit{optimistic follow-the-regularized-leader (OFTRL)} framework,
which is also used in some existing studies \citep{wei2018more,ito2021parameter}.
This algorithm computes a probability vector $p(t) = (p_1(t), p_2(t), \ldots, p_K(t))^\top \in \simplex{K} = \{ p \in [0, 1]^K \mid \| p \|_1 = 1 \}$ 
and then chooses $I(t) \in [K]$ following $p(t)$
(i.e., so that $\prob [I(t) = i | p(t)] = p_i(t)$).
The OFTRL update rule
is expressed as
\begin{align}
  \label{eq:defOFTRL}
  p(t) \in \argmin_{p \in \simplex{K}} \left\{
    \linner
    m(t)
    +
    \sum_{s=1}^{t-1}
    \hat{\ell}(s),
    p
    \rinner
    +
    \psi_t (p) 
  \right\},
\end{align}
where $m(t) \in [0, 1]^K$ corresponds to an optimistic prediction of the true loss vector $\ell(t)$,
$\hat{\ell}(t) \in \re^K$ is an unbiased estimator of $\ell(t)$,
and $\psi_t(p)$ is a regularizer
function that is
a convex function over $\simplex{K}$.
We choose the optimistic prediction to be the empirical mean of
the
losses so far,
i.e.,
we define $m(t) = (m_1(t), \ldots, m_K(t))^\top \in [0, 1]^K$ by
\begin{align}
  m_i(t) = \frac{1}{1 + N_i(t-1) } \left( \frac{1}{2} + 
  \sum_{s=1}^{t-1}\mathbf{1}[I(s) = i] \ell_i(s) \right)  ,
  \label{eq:defmi}
\end{align}
where $N_i(t)$ is the number of times the arm $i$ is chosen until
the $t$-th round,
i.e.,
$N_i(t) = |\{ s \in [t] \mid I(s) = i \}|$.
We use an unbiased estimator $\hat{\ell}(t) = ( \hat{\ell}_1(t), \ldots, \hat{\ell}_K(t) )^\top \in \re^K$
given as follows:
\begin{align}
  \label{eq:defellhat}
  \hat{\ell}_i(t) = m_i(t) + \frac{\mathbf{1}[ I(t) = i ]}{ p_i(t) } (\ell_i(t) - m_i(t)) .
\end{align}
This is indeed an unbiased estimator of ${\ell}(t)$
since
$
\E[ \hat{\ell}_i(t) | p(t) ] =
m_i(t) + \frac{p_i(t)}{p_i(t)} (\ell_i(t) - m_i(t))
=
\ell_i(t)
$.
The optimistic prediction $m(t)$ in \eqref{eq:defellhat} has a role in reducing the variance of $\hat{\ell}(t)$;
the better $m(t)$ predicts $\ell(t)$,
the smaller the variance in $\hat{\ell}(t)$.
The regularizer function in this work is given as
\begin{align}
  \label{eq:defpsi}
  \psi_t(p) 
  = \sum_{i=1}^K \beta_i(t) \phi (p_i) ,
  \quad
  \mbox{where}
  \quad
  \phi(x) = 
  x - 1 - \log x + \gamma \cdot ( x + ( 1 - x) \log (1 - x ) )
  \quad
\end{align}
with $\gamma = \log T$ and the regularization parameters $\beta_i(t) \geq 0$.
Our regularizer in \eqref{eq:defpsi} consists of parts of the logarithmic barrier ($- \log p_i$),
the entropy term ($(1-p_i)\log (1-p_i)$) for the complement of $p_i$,
and affine functions for minor adjustments.\footnote{
  Affine parts have been added to ensure 
  $0 \leq \phi(x) \leq \log T$
  for $x \in (1/T, 1]$,
  which makes the analysis simpler and leads to smaller constant factors than in existing studies,
  e.g.,
  by \citet{wei2018more,ito2021hybrid,ito2021parameter}.
}
Regularization parameters $\beta_i(t) $ are updated as
\begin{align}
  &
  \beta_i(t) = \sqrt{(1 + \epsilon)^2 + \frac{1}{ \gamma } \sum_{s=1}^{t-1} \alpha_i(s) } ,\nonumber\\
&  \phantom{wwwwwwwww}
  \mbox{where} \quad
  \alpha_i(t) 
  =
  \mathbf{1}[I(t)=i]( \ell_i(t) - m_i(t) )^2 \min \left\{ 1 ,\frac{2(1 - p_i(t))}{p_i(t)^2 \gamma} \right\} 
  \label{eq:defalphabeta}
\end{align}
and $\epsilon \in (0, 1/2]$ is an input parameter.
The parameter $\alpha_i(t)$ in \eqref{eq:defalphabeta} is 
defined so that the \textit{stability term}
(discussed in Section~\ref{sec:OFTRL})
in the regret can be bounded by an $O( \sum_{i=1}^K \alpha_i(t) / \beta_i(t))$-term.
The factor $(1-p_i(t))$ in $\alpha_i(t)$ is due to the entropy portion 
of the regularizer function defined in \eqref{eq:defpsi},
which allows us to remove the influence of terms related to the optimal arm in our regret analysis,
as we will see in Section~\ref{sec:analysis-sto}.
For each suboptimal arm $i \in [K] \setminus \{ i^* \}$,
we have $\alpha_i(t) = \mathbf{1}[I(t) = i](\ell_i(t) - m_i(t))^2$ as $p_i(t)$ is far from $1$ in most cases.
This means that $\sum_{s=1}^{t-1} \alpha_i(s)$,
which appears in the definition of the regularization parameter $\beta_i(t)$,
corresponds to the cumulative error in predicting $\ell_i(s)$ with $m_i(s)$,
which can be linked to the variance of loss vectors if $m_i(t)$ is given by \eqref{eq:defmi}.
This connection will be used to show the variance-dependent regret bounds in Section~\ref{sec:analysis-sto}.

\section{Regret Analysis for the Proposed Algorithm} 
This section provides a regret analysis for the proposed algorithm
to show bounds in Theorem~\ref{thm:main}.

\subsection{Overview of the analysis}
In the following subsections,
we first show a regret bound dependent on $\{ \alpha_i(t) \}$ expressed as
\begin{align}
  \label{eq:approxboundOFTRL}
  R(T) = O \left(
    \sum_{i=1}^K
    \E
    \left[
    \sqrt{ \sum_{t=1}^T \alpha_i(t)  \log T}
    \right]
    +
    K \log T
  \right),
\end{align}
which corresponds to Lemma~\ref{lem:boundRTbeta} below.
To show Lemma~\ref{lem:boundRTbeta} and \eqref{eq:approxboundOFTRL},
we use a standard analysis technique for OFTRL 
to decompose the regret into the \textit{penalty terms} and the \textit{stability terms}.
We give an upper bound on each of them in terms of $\{ \alpha_i(t) \}$,
with the aid of a specific structure of the regularizer function given in \eqref{eq:defpsi}
and the construction of the regularization parameters $\{ \beta_i(t) \}$ in \eqref{eq:defalphabeta}.

In a stochastic setting,
we evaluate the terms in \eqref{eq:approxboundOFTRL}
separately for the optimal arm $i^*$ and
for the other suboptimal arms $i \in [K] \setminus \{ i^* \}$.
For suboptimal arms,
as $\alpha_i(t) \leq \mathbf{1}[I(t) = i] (\ell_i(t) - m_i(t))^2$ follows from the definition
and $m_i(t)$ converges to $\mu_i$,
each term in \eqref{eq:approxboundOFTRL} can be bounded as
$
  \E \left[ \sum_{t=1}^T \alpha_i(t) \right] 
  \lesssim 
  \E \left[ \sum_{t=1}^T \mathbf{1}[ I(t) = i ] (\ell_i(t) - \mu_i)^2 \right]
  =
  \sigma_i^2 P_i
$,
where $P_i \in [0, T]$ is defined as the expected number of times the arm $i$ is chosen,
i.e.,
\begin{align}
  \label{eq:defPi}
  P_i 
  =
  \E \left[ N_i(T) \right]
  =
  \E \left[
    \sum_{t=1}^T \mathbf{1}[ I(t) = i ] 
  \right]
  =
  \E \left[
    \sum_{t=1}^T p_i(t)
  \right].
\end{align}
For the optimal arm $i^*$,
we give a bound on $\alpha_{i^*}(t)$ as
$
  \E\left[\sum_{t=1}^T\alpha_{i^*}(t)\right] \leq \E \left[ \sum_{t=1}^T (1 - p_{i^*}(t)) \right]/\sqrt{\gamma}
  = \sum_{i \neq i^*} P_i /\sqrt{\gamma} 
$,
which is asymptotically negligible
due to the condition of $\gamma = \log T$.
Consequently,
we obtain a 
$\{ \sigma_i^2 P_i \}$-dependent regret bound of
$
  R(T) = 
  O \left( \sum_{i \neq i^*}
  \sqrt{\sigma_i^2 P_i \log T}
  +
  K \log T
  \right)
$,
which implies an $O(\log T)$-regret bound with the aid of the self-bounding technique.
We use the fact that $R(T)$ can be expressed as $R(T) = \sum_{i\neq i^*} \Delta_i P_i $ in a stochastic setting.
By combining this with the above $\{ \sigma_i^2 P_i \}$-dependent regret bound,
we obtain
\begin{align}
  \nonumber
  R(T)
  &
  = 2 R(T) - R(T) 
  \\
  &
  =
  O \left( \sum_{i \neq i^*}
  \left( \sqrt{\sigma_i^2 P_i \log T} - \Delta_i P_i \right)
  +
  K \log T
  \right)
  =
  O \left( \sum_{i \neq i^*}
  \frac{\sigma_i^2 \log T}{\Delta_i} 
  +
  K \log T
  \right)
  \label{eq:approxboundsto}
\end{align}
where the last equality follows from $2\sqrt{ax} - bx = - (\sqrt{a/b} - \sqrt{bx})^2 + a/b \leq a/b$ that holds for any $a, x \geq 0$, and
$b>0$.

For the adversarial setting,
we evaluate the RHS of \eqref{eq:approxboundOFTRL} via inequalities
of
$
  \sum_{i=1}^K \sqrt{ \sum_{t=1}^T \alpha_i(t) }
  \leq
  \sqrt{ K\sum_{i=1}^K \sum_{t=1}^T \alpha_i(t) }
  \leq
  \sqrt{ K\sum_{t=1}^T ( \ell_{I(t) }(t) - m_{I(t)}(t) )^2 }
$,
where we have used the Cauchy-Schwarz inequality and
the second inequality follows from the definition of $\alpha_i(t)$ given in \eqref{eq:defalphabeta}.
Further,
for $m(t)$ defined by \eqref{eq:defmi},
we have
$ \sum_{t=1}^T ( \ell_{I(t)}(t) - m_{I(t)}(t) )^2 \leq 
\sum_{t=1}^T ( \ell_{I(t)}(t) - m_{I(t)}^* )^2 + O(K \log T)
$ for any $m^* \in [0, 1]^K$,
which follows from a standard regret analysis for online linear regression with square loss
(see, e.g., Chapter 11 of the book by \citet{cesa2006prediction}).
Combining these with \eqref{eq:approxboundOFTRL},
we obtain
$R(T) = O\left( \sqrt{ \E [ K \sum_{t=1}^T ( \ell_{I(t)}(t) - m_{I(t)}^* )^2  ] \log T} + K \log T \right)$
for any $m^*$,
which implies a regret bound in the form of \eqref{eq:thm1-bound-adv}.

\subsection{Analysis for OFTRL}
\label{sec:OFTRL}
Now we introduce the sketch of the analysis given above in more details.
Let us start here with the standard
analysis for OFTRL. 
Let $D_t$ denote the Bregman divergence associated with $\psi_t$,
i.e.,
  $
  D_t(q, p) =
  \psi_t(q) - \psi_t(p) - \linner \nabla \psi_t(p), q - p \rinner
  $.
The regret for OFTRL can then be bounded as follows:
\begin{lemma}
  \label{lem:boundOFTRL}
  If $p(t)$ is given by \eqref{eq:defOFTRL},
  for any $p^* \in \simplex{K} \cap \re^K_{>0}$,
  we have
  \begin{align}
    \nonumber
    \sum_{t=1}^T \linner 
      \hat{\ell}(t), p(t) - p^*
    \rinner
    &
    \leq
    \uline{
      \psi_{T+1}(p^*)
      -
      \psi_1( q(1) )
      +
      \sum_{t=1}^T
      (\psi_{t}(q(t+1))
      - \psi_{t+1}(q(t+1)) 
      )
    }
    \\
    &
    \quad
    +
    \uwave{
    \sum_{t=1}^T
    \left(
      \linner \hat{\ell}(t) - m(t), p(t) - q(t+1) \rinner
      - D_t(q(t+1), p(t)) 
    \right)
    }
    \label{eq:boundOFTRL}
  \end{align}
  where 
  $q(t)$ is defined as
  $
    q(t) \in \argmin_{p \in \simplex{K}} \left\{ 
      \linner \sum_{s=1}^{t-1} \hat{\ell}(s), p \rinner + \psi_t(p) 
    \right\}
  $.
\end{lemma}
All omitted proofs are given in the appendix.
We refer to the parts
\uline{underscored with a straight line} and
\uwave{underscored with a wavy line} in \eqref{eq:boundOFTRL}
{\textit{penalty terms}} and {\textit{stability terms}},
respectively.
In our analysis,
we use this lemma with $p^* = (1 - \frac{K}{T}) \chi_{i^*} + \frac{1}{T}\mathbf{1}$,
where $\chi_{i^*} \in \{ 0 , 1 \}^K$ is the indicator vector associated with the optimal arm $i^*$.
Each of the
{{penalty terms}} and {{stability terms}}
can then be bounded as follows.

\paragraph{\uline{Penalty terms}}
From the definition of $\phi(x)$ in \eqref{eq:defpsi},
we have $\phi(x) \geq 0$ for any $x \in (0, 1]$
and $\phi(x) \leq \log T = \gamma$ for any $x \in (1/T, 1 ]$.
Hence,
as the regularizer $\psi_t(p)$ is defined as 
$\sum_{i=1}^K \beta_i(t) \phi(p_i)$,
the first part of the penalty terms are bounded as
$\psi_{T+1}(p^*) - \psi_1(q(1)) \leq \psi_{T+1}(p^*) \leq \gamma \sum_{i=1}^K \beta_i(T+1) $.
Further,
as $\beta_i(t)$ given in \eqref{eq:defalphabeta}
is monotonically non-decreasing in $t$,
the remaining parts in the penalty terms is non-positive.
Hence,
the penalty terms can be bounded by $\gamma \sum_{i=1}^K \beta_i(T+1)$.

\paragraph{\uwave{Stability terms}}
From the logarithmic-barrier part and the entropy part of the regularizer function,
the stability terms for the $t$-th round can be bounded as
$O\left( \frac{(\ell_{i}(t) - m_{i}(t) )^2}{ \beta_i(t) } \right)$ and
$O\left( \frac{(\ell_{i}(t) - m_{i}(t) )^2 (1-p_{i})}{\beta_i(t) p_i(t)^2 \gamma} \right)$,
respectively,
with $i = I(t)$.
This means that the stability terms for the $t$-th round can be
bounded by an $O(\sum_{i=1}^K \alpha_i(t)/\beta_i(t))$-term with $\alpha_i(t)$ defined in \eqref{eq:defalphabeta}.
By taking the sum of this for all rounds $t \in [T]$ and using the definition of $\beta_i(t)$ in \eqref{eq:defalphabeta},
we can show
$
  O(\sum_{t=1}^T \sum_{i=1}^K \alpha_i(t) / \beta_i(t))
  \leq
  O(\gamma \sum_{i=1}^K \beta_i(T+1) )
$.

By formalizing the above discussion,
we can obtain the following regret bound:
\begin{lemma}
  \label{lem:boundRTbeta}
  For any environment,
  the regret for the proposed algorithm is bounded as
  \begin{align}
    R(T)
    \leq
    \gamma
    \sum_{i=1}^K
    \E \left[
      2  \beta_i(T+1)
      -
      \beta_i(1)
      +
      2
      \delta \log  \frac{\beta_{i}(T+1)}{\beta_i(1)}
    \right]
    +
    2K ( 1 + \delta ),
  \end{align}
  where $\delta > 0$ is defined by
  $\delta = (1+ \epsilon)^3 \log \frac{1+\epsilon}{\epsilon} - (1+\epsilon)^2 - \frac{1+\epsilon}{2} =
  O(\log \frac{1}{\epsilon})$.
\end{lemma}
As we have $\log (x + 1) \leq \sqrt{x}$,
this lemma implies that $R(T) = O \left( \gamma \sum_{i=1}^K \E [ \beta_i(T+1) ] + K \right) $.
Further,
from 
the condition of $\gamma = \log T$ and the definition of $\beta_i(t)$ in \eqref{eq:defalphabeta},
we can see that \eqref{eq:approxboundOFTRL} holds.

\subsection{Stochastic setting (with adversarial corruption) 
\inote{added (with adv...)}
}
\label{sec:analysis-sto}
This subsection shows \eqref{eq:thm1-bound-sto} by using Lemma~\ref{lem:boundRTbeta} and the self-bounding technique.
Lemma~\ref{lem:boundRTbeta} combined with \eqref{eq:defalphabeta}
implies that the regret can be bounded as \eqref{eq:approxboundOFTRL}
in terms of $\sum_{t=1}^T \alpha_i(t)$.
We analyze $\sum_{t=1}^T \alpha_i(t)$,
separately for the optimal arm $i^* \in [K]$ and suboptimal arms $i \in [K] \setminus \{ i^* \}$.

Recall that
  $\alpha_i(t)$ and $m_i(t)$ are given by \eqref{eq:defalphabeta} and \eqref{eq:defmi},
  respectively.
For suboptimal arms $i \in [K] \setminus \{ i^* \}$,
we give a bound on $\sum_{t=1}^T \alpha_i(t)$ via the following lemma:
\begin{lemma}
  \label{lem:bound-sumalpha-suboptimal}
  It holds for any $i \in [K]$ and $m_i^* \in [0, 1]$ that
  \begin{align*}
    \sum_{t=1}^T \alpha_i(t)
    \leq
    \sum_{t=1}^{T}\mathbf{1}[I(t) = i] (\ell_i(t) - m_i(t) )^2 
    \leq
    \sum_{t=1}^{T}\mathbf{1}[I(t) = i] (\ell_i(t) - m_i^* )^2 
    + \log (1 + N_i(T))  + \frac{5}{4} .
  \end{align*}
\end{lemma}
From this lemma,
in the stochastic setting,
we have
\begin{align}
  \E \left[
    \sum_{t=1}^T \alpha_i(t)
  \right]
  \leq
  \E \left[
    \sum_{t=1}^{T}p_i(t) \sigma_i^2
  + \log (1 + N_i(T)) 
  \right]
  + \frac{5}{4} 
  \leq
  \sigma_i^2 P_i + \log (1 + P_i) + \frac{5}{4} ,
  \label{eq:bound-sumalpha-suboptimal}
\end{align}
where the first inequality follows from Lemma~\ref{lem:bound-sumalpha-suboptimal} with
$m^*_i = \mu_i$ and
the last inequality follows from the definition \eqref{eq:defPi} of $P_i$
and Jensen's inequality.

For the optimal arm $i^*$,
we give a bound on $\sum_{t=1}^T \alpha_i(t)$ via the following lemma:
\begin{lemma}
  \label{lem:bound-sumalpha-optarm}
  It holds for any $i^* \in [K]$ that
  \begin{align}
    \label{eq:bound-sumalpha-optarm}
    \E
    \left[
      \sum_{t=1}^T \alpha_{i^*}(t)
    \right]
    \leq
    \E
    \left[
      \sum_{t=1}^T \mathbf{1}[I(t) = i^*] \min \left\{
        1, 
        \frac{2(1 - p_{i^*}(t))}{\gamma p_{i^*}(t)^2}
      \right\}
    \right]
    \leq
    \frac{2}{\sqrt{\gamma}}
    \sum_{i \in [K] \setminus \{ i^* \}}
    P_i.
  \end{align}
\end{lemma}
In the proof of this lemma,
we use the upper bound for $\alpha_{i^*}(t)$ depending on $(1-p_{i^*}(t))$,
which is due to the entropy part of the regularizer functions.

Combining \eqref{eq:approxboundOFTRL}
with \eqref{eq:bound-sumalpha-suboptimal} for $i \in [K] \setminus \{ i^* \}$
and \eqref{eq:bound-sumalpha-optarm},
we obtain a regret bound of
\begin{align}
  \nonumber
  R(T)
  &
  = 
  O \left(
    \left(
    \sum_{i \in [K] \setminus\{ i^* \} }
    \sqrt{ \sigma_i^2 P_i}
    +
    \sqrt{ \sum_{i \in [K] \setminus \{ i^* \}} P_i / \sqrt{\gamma}}
    \right)
    \sqrt{
    \log {T}}
    +
    K \log T
  \right)
  \\
  &
  =
  O \left(
    \sum_{i \in [K] \setminus\{ i^* \} }
    \sqrt{ \tilde{\sigma}_i^2 P_i \log T}
    +
    K \log T
  \right),
  \quad
    \mbox{where}
    ~
    \tilde{\sigma}_i
    =
    \sigma_i 
    +
    O\left(\gamma^{-1/4}\right).
  \label{eq:approxboundOFTRL-sto}
\end{align}
We obtain
$
  R(T) = 2 R(T) - R(T) 
  =
  O \left( \sum_{i \neq i^*}
  \frac{\tilde{\sigma}_i^2 \log T}{\Delta_i} 
  +
  K \log T
  \right)
$ by combining \eqref{eq:approxboundOFTRL-sto} with the discussion in \eqref{eq:approxboundsto}.
By further refining the above analysis,
from Lemmas~\ref{lem:boundRTbeta}, \ref{lem:bound-sumalpha-suboptimal} and \ref{lem:bound-sumalpha-optarm},
we can obtain the regret bound in \eqref{eq:thm1-bound-sto}.
The regret bound for corrupted stochastic settings can also be shown 
from a similar $\{\sigma_i^2 P_i \}$-dependent regret bound.
In the proof,
we use the equation of $R(T) = (1+\lambda)R(t) - \lambda R(T)$ with
an appropriately chosen $\lambda \in (0, 1]$,
as considered in some existing studies,
e.g.,
by \citet{zimmert2021tsallis}.
The details are given in Appendix~\ref{sec:appendix-corrupted}.

\subsection{Adversarial setting}
We can obtain the regret bound in \eqref{eq:thm1-bound-adv} for the adversarial setting
in the following way.
We first have
$
  R(T)
  \leq
  2
  \sum_{i=1}^K
  \E \left[
    \sqrt{\sum_{t=1}^T \alpha_i(t) \log T}
  \right]
  +
  O(K \log T)
  \leq
  2
  \E \left[
    \sqrt{
      K
      \sum_{i=1}^K
      \sum_{t=1}^T \alpha_i(t) \log T}
  \right]
  +
  O(K \log T)
  \leq
  2
  \E \left[
    \sqrt{
      K
      \sum_{t=1}^T ( \ell_{I(t)} - m^*_{I(t)}  )^2 \log T}
  \right]
  +
  O(K \log T)
$
for any $m^* \in [0,1]^K$,
where the first inequality follows from Lemma~\ref{lem:boundRTbeta} and \eqref{eq:defalphabeta},
the second inequality is due to the Cauchy-Schwarz inequality,
and the last inequality comes from Lemma~\ref{lem:bound-sumalpha-suboptimal}.
From $
  \sum_{t=1}^T ( \ell_{I(t)}(t) - m^*_{I(t)}  )^2 
  \leq
  \sum_{t=1}^T \| \ell(t) - m^* \|_{\infty}^2
$,
we have
the regret bound of 
$R(T) \leq 2 \sqrt{K Q_{\infty} \log T }  + O(K \log T)$,
where $Q_{\infty} = \min_{\bar{\ell} \in \re^K} \E [\sum_{t=1}^T \| \ell(t) - \bar{\ell} \|_{\infty}^2]$.
As $Q_{\infty} \leq T / 4$ follows from the definition,
this $Q_{\infty}$-dependent bound immediately implies
$R(T) \leq  \sqrt{K T \log T }  + O(K \log T)$.
Further,
we can show the regret bounds dependent on $L^* = \min_{i^* \in [K]} \E [ \sum_{t=1}^T \ell_{i^*}(t) ]$
via a similar discussion within existing studies,
e.g., by \citet{wei2018more}.
By considering the cases of $m^* = 0$ and $m^* = \mathbf{1}$,
we can obtain
$R(T) \leq 2 \sqrt{K L^* \log T }  + O(K \log T)$
and 
$R(T) \leq 2 \sqrt{K (T - L^*) \log T }  + O(K \log T)$,
respectively.
Regret bounds in this subsection are summarized in \eqref{eq:thm1-bound-adv}
(see Appendix \ref{sec:appendix-adv} for details).
%

\section{Regret Lower Bound for Stochastic Setting}
\newcommand{\n}{\nonumber}
\newcommand{\nn}{\nonumber\\}
\newenvironment{proof2}[1]{\noindent{\bf #1\,}}{\vspace{0.5mm}}

Whereas
the regret bound of form
$\limsup_{T\to\infty}\frac{R(T)}{\log T} \le \sum_{i:\neq i^*}a \frac{\sigma_i^2}{\Delta_i} + b$
inevitably involves
positive $a$ and $b$ \citep[Exercise16.7]{lattimore2020bandit},
there has not been a formal discussion on the best possible constants of this kind of form.
On the other hand,
\cite{honda_moment} derived the tight regret bound in terms of the moments of reward distributions for any consistent algorithms.
Here an algorithm is called consistent if it satisfies $R(T)=O(T^a)$ for any fixed $a \in (0, 1)$ and $\mathcal{D}$ \citep{lai1985asymptotically}.
Though the original bound derived there is quite complicated (see \eqref{bound_dinf}--\eqref{nu_expression} in Appendix~\ref{append_lower}),
the bound is significantly simplified as shown in the following proposition,
the derivation of which is given in Appendix~\ref{append_lower}.
\begin{proposition}\label{prop_moment}
For any $0\le \mu^*<\mu_i\le 1$ and $\sigma_i^2 \le \mu_i(1-\mu_i)$,
there exists a loss distribution $\mathcal{D}$ over $[0,1]^K$ such that any consistent algorithm suffers
\begin{align}
\liminf_{T\to\infty}\frac{R(T)}{\log T}
&\ge
\sum_{i\neq i^*}
\frac{\Delta_i}{
\frac{\mu_i^2}{\sigma_i^2+\mu_i^2}\log \frac{\mu_i}{\mu^*}
+\frac{\sigma_i^2}{\sigma_i^2+\mu_i^2}\log \frac{\sigma_i^2}{\sigma_i^2+\mu_i\Delta_i}}
\label{lower_simple}
\end{align}
and there exists an algorithm achieving this bound with equality for all $\mathcal{D}$.
\end{proposition}
Note that the expected loss $\mu_i$ in this proposition corresponds to $1-\mu_i$ under the notation of \cite{honda_moment}
since they considered cumulative reward maximization with reward in $[0,1]$.
Still, replacing $\mu_i$ with $1-\mu_i$ in \eqref{lower_simple} makes the expression considerably
messy, even though most stochastic bandit papers following \cite{lai1985asymptotically}
formulate the problem as reward maximization unlike adversarial bandit papers considering loss minimization.
This observation may partly suggest that
the formulation based on loss minimization is also natural for the stochastic settings.
%

By this proposition,
the best-possible bound of form
$\sum_{i\neq i^*}f(\Delta_i,\,\sigma_i^2)$ is expressed as\footnote{%
If we optimize $\mu^*$ outside the summation then a tighter bound is obtained
at the cost of the simplicity.}
\begin{align}
\limsup_{T\to\infty}\frac{R(T)}{\log T}
&\le
\sum_{i\neq i^*}
\sup_{\substack{(\mu_i,\mu^*)\in[0,1]^2:\\ \mu_i-\mu^*=\Delta,\, \sigma_i^2 \le \mu_i(1-\mu_i)}}
\frac{\Delta_i}{
\frac{\mu_i^2}{\sigma_i^2+\mu_i^2}\log \frac{\mu_i}{\mu^*}
+\frac{\sigma_i^2}{\sigma_i^2+\mu_i^2}\log \frac{\sigma_i^2}{\sigma_i^2+\mu_i\Delta_i}}
\nn
&\approx\sum_{i\neq i^*}
\left(2\frac{\sigma_i^2}{\Delta_i}+\frac12\log \left(1+\frac{\sigma_i^2}{\Delta_i}\right)+1\right),\label{bound_approx}
\end{align}
where the relative error of the approximation in \eqref{bound_approx}
goes to 0 as
$\sigma_i^2/\Delta_i\to 0,\infty$.
We can numerically verify that the relative approximation error is
at most\footnote{The relative error is refined to $\pm 0.6\%$ if we replace $\frac12\log(1+\frac{\sigma_i^2}{\Delta_i})$
in \eqref{bound_approx} with $0.06\log(1+\frac{60\sigma_i^2}{\Delta_i})$.} $\pm6\%$ for all $\sigma_i^2/\Delta_i>0$.
From this expression we can see that
the derived regret bound \eqref{eq:thm1-bound-sto} is close to twice or $2(1+\epsilon)$ times this best bound in most cases.

\section{Conclusion}
In this paper,
we provided an MAB algorithm with a variance-dependent regret bound, adversarial robustness,
and data-dependent regret bounds in adversarial environments.
The (stochastic) variance-dependent regret upper bound shown in this paper
is close to twice the lower bound.
Removing this gap of twice is one important challenge that even 
variance-unaware adversarially robust algorithms have not yet overcome
(e.g., Tsallis-INF~\citep{zimmert2021tsallis} has a gap of twice as well).
One natural question would be
whether we can improve performance in stochastic settings without sacrificing the adversarial robustness. 


\bibliographystyle{abbrvnat}
\bibliography{reference.bib}

\newpage
\appendix


\section{Omitted Proofs}
This section provides proofs omitted in the main text.
We use the notation of $\bet = (1+\epsilon)$.
We then have $\beta_i(1) = \bet$ for any $i \in [K]$.
\subsection{Proof of Lemma~\ref{lem:boundOFTRL}}
\textbf{Proof of Lemma~\ref{lem:boundOFTRL}}\quad
From
the definitions of $p(t)$ and $q(t)$,
we have
\begin{align*}
  &
  \linner \sum_{t=1}^T \hat{\ell}(t), p^* \rinner + \psi_{T+1}(p^*)
  \\
  &
  \geq
  \linner \sum_{t=1}^T  \hat{\ell}(t), q(T+1) \rinner
  + \psi_{T+1}(q(T+1))
  + D_{T+1}(p^*, q(T+1))
  \\
  &
  =
  \linner m(T) + \sum_{t=1}^{T-1} \hat{\ell}(t), q(T+1) \rinner + \linner \ell(T) - m(T), q(T+1) \rinner 
  \\
  &
  \quad
  + D_{T+1}(p^*, q(T+1)) + \psi_{T+1}(q(T+1))
  \\
  &
  \geq
  \linner m(T) + \sum_{t=1}^{T-1} \hat{\ell}(t), p(T) \rinner
  + \psi_{T}(p(T))
  + D_{T}(q(T+1), p(T) )
  - \psi_T(q(T+1))
  \\
  &
  \quad
  + \linner \hat{\ell}(T) - m(T), q(T+1) \rinner + D_{T+1}(p^*, q(t+1)) + \psi_{T+1}(q(T+1))
  \\
  &
  =
  \linner \sum_{t=1}^{T-1} \hat{\ell}(t), p(T) \rinner
  + \psi_{T}(p(T))
  + D_{T}(q(T+1), p(T) )
  - \psi_T(q(T+1))
  \\
  &
  \quad
  + \linner m(T) , p(T) \rinner + \linner \hat{\ell}(T) - m(T), q(T+1) \rinner + D_{T+1}(p^*, q(t+1)) + \psi_{T+1}(q(T+1))
  \\
  &
  \geq
  \linner \sum_{t=1}^{T-1} \hat{\ell}(t), q(T) \rinner
  + \psi_{T}(q(T))
  + D_{T}(q(T+1), p(T) )
  - \psi_T(q(T+1))
  \\
  &
  \quad
  + \linner m(T) , p(T) \rinner + \linner \hat{\ell}(T) - m(T), q(T+1) \rinner + D_{T+1}(p^*, q(t+1)) + \psi_{T+1}(q(T+1))
  \\
  &
  \geq
  \sum_{t=1}^T
  \left(
    \linner m(t), p(t) \rinner  + \linner \hat{\ell}(t) - m(t), q(t+1) \rinner
  \right)
  \\
  &
  \quad
  +
  \sum_{t=1}^T
  \left(
    D_t(q(t+1), p(t))
    + \psi_{t+1}(q(t+1)) - \psi_{t}(q(t+1))
  \right)
  +
  \psi_1( q(1) ),
\end{align*}
where the first, the second, and the third inequalities follow from
the definition of $p_i(t)$ and $q_i(t)$,
and the last inequality can be obtained by repeating the same transformation $T$ times.
From this,
we have
\begin{align*}
  \sum_{t=1}^T \linner
    \hat{\ell}(t), p(t) - p^*
  \rinner
  &
  \leq
  \psi_{T+1}(p^*)
  -
  \psi_1( q(1) )
  +
  \sum_{t=1}^T
  (\psi_{t}(q(t+1))
  - \psi_{t+1}(q(t+1))
  )
  \\
  &
  \quad
  +
  \sum_{t=1}^T
  \left(
    \linner \hat{\ell}(t) - m(t), p(t) - q(t+1) \rinner
    - D_t(q(t+1), p(t))
  \right) ,
\end{align*}
which completes the proof of Lemma~\ref{lem:boundOFTRL}.
\qed

\subsection{Proof of Lemma~\ref{lem:boundRTbeta}}
Define
\begin{align*}
  \delta
  =
  \bet^3
  \left(
    -
    \log \left(
      1 - \frac{1}{\bet}
    \right)
    -
    \frac{1}{\bet}
    -
    \frac{1}{2 \bet^2}
  \right)
  =
  (1+\epsilon)^{3} \log\frac{1+ \epsilon}{\epsilon}
  -
  (1+ \epsilon)^2
  -
  \frac{1+\epsilon}{2} .
\end{align*}
Then,
for any $x \leq 1 / \bet$,
we have
\begin{align*}
  \frac{
  - \log(1 - x) - x - x^2/2
  }{
    |x|^3
  }
  \leq
  \delta
\end{align*}
holds,
which implies
\begin{align}
  x \leq 1 / \bet
  \quad
  \Longrightarrow
  \quad
  - \log( 1 - x )
  - x
  - x^2 / 2
  \leq
  \delta |x|^3 .
  \label{eq:boundgg}
\end{align}
This inequality will be used in the proof of Lemma~\ref{lem:boundRTbeta}.

We use the following lemma to analyze the stability terms in \eqref{eq:boundOFTRL}:
\begin{lemma}
  \label{lem:boundD12}
  Let $D^{(1)}$ and $D^{(2)}$ denote the Bregman divergences associated with $\phi^{(1)}(x) = - \log x$
  and $\phi^{(2)}(x) = (1-x) \log (1-x)$.
  Then,
  for any $x \in (0,1)$,
  we have
  \begin{align}
    \label{eq:boundf1}
    \max_{y \in \re} f^{(1)}(y)
    &
    :=
    \max_{y \in \re} \left\{ a (x - y) - D^{(1)}(y, x) \right\} =
    g(ax)
    &
    \quad
    (a \geq - 1/x),
    \\
    \max_{y \in \re} f^{(2)}(y)
    &
    :=
    \max_{y \in \re} \left\{ a (x - y) - D^{(2)}(y, x) \right\} =
    (1-x) h(a)
    &
    \quad
    (a \in \re),
    \label{eq:boundf2}
  \end{align}
  where $g$ and $h$ are defined as
  \begin{align}
    \label{eq:defgh}
    g(x) = x - \log (x + 1) ,
    \quad
    h(x) = \exp(x) - x - 1   .
  \end{align}
\end{lemma}
\begin{proof}
  The derivative of $f^{(1)}$ is expressed as
  \begin{align*}
    \frac{\mathrm{d}f^{(1)}}{\mathrm{d}y}(y) = - a + \frac{1}{y} - \frac{1}{x}.
  \end{align*}
  As $f^{(1)}$ is a concave function,
  the maximizer $y^*$ of $f^{(1)}$ satisfies $-a + \frac{1}{y^*} - \frac{1}{x} = 0$.
  Hence,
  the maximum value is expressed as
  \begin{align*}
    \max_{y \in \re} f^{(1)}(y)
    &
    =
    f^{(1)}(y^*)
    =
    a(x-y^*) + \log y^* - \log x + \frac{x - y^*}{x}
    \\
    &
    =
    - \log \frac{x}{y^*}
    + \frac{x - y^*}{y^*}
    =
    - \log (1+ax)
    + a x
    =
    g(ax),
  \end{align*}
  which proves \eqref{eq:boundf1}.
  Similarly,
  as $f^{(2)}$ is a concave function,
  the maximizer $y^* \in \re$ of $f^{(2)}$ satisfies
  \begin{align*}
    \frac{\mathrm{d}f^{(2)}}{\mathrm{d}y}(y^*) = - a + \log(1-y^*) - \log(1-x) = 0.
  \end{align*}
  Hence,
  we have
  \begin{align*}
    f^{(2)}(y^*)
    &
    =
    a (x - y^*) - (1-y^*)\log(1-y^*) + (1-x)\log(1-x) - ( \log (1-x) + 1 )(y^* - x)
    \\
    &
    =
    (1 - y^*) - (1-x) - (1 - x)\log(1-y^*) + (1-x)\log(1-x)
    \\
    &
    =
    (1-x)(\exp(a) - 1) - (1-x) a
    =
    (1-x)(\exp(a) - a - 1),
  \end{align*}
  which proves \eqref{eq:boundf2}.
\end{proof}
In addition to this lemma,
we use the following bounds on $g$ and $h$:
\begin{align}
    g(x) & = x - \log (x + 1) \leq \frac{1}{2}x^2
    + \delta |x|^3 \quad \left( x \geq - \frac{1}{\bet} \right),
    \label{eq:boundg}
    \\
    h(x) & = \exp(x) - x - 1 \leq x^2
    \quad \left( x \leq 1 \right).
    \label{eq:boundhh}
\end{align}
We can easily see that
\eqref{eq:boundg} follows from \eqref{eq:boundgg}.
The bound on $h(x)$ in \eqref{eq:boundhh} can be shown from Taylor's theorem.

\quad\\
\textbf{Proof of Lemma~\ref{lem:boundRTbeta}}\quad
  Define $p^* = (1 - \frac{K}{T}) \chi_{i^*} + \frac{1}{T}\mathbf{1}$,
  where $\chi_{i^*} \in \{ 0, 1 \}^K$ denotes the indicator vector of $i^* \in [K]$,
  i.e.,
  $\chi_{i^*, i} = 1$ if $i = i^*$ and $\chi_{i^*, i} = 0$ if $i \neq i^*$.
  We then have
  \begin{align}
    \nonumber
    R_{i^*}(T)
    &
    =
    \E \left[
    \sum_{t=1}^T
    \linner
    \ell(t), p(t) - \chi_{i^*}
    \rinner
    \right]
    =
    \E \left[
    \sum_{t=1}^T
    \linner
    \ell(t), p(t) - p^*
    \rinner
    +
    \sum_{t=1}^T
    \linner
    \ell(t), p^* - \chi_{i^*}
    \rinner
    \right]
    \\
    &
    =
    \E \left[
    \sum_{t=1}^T
    \linner
    \hat{\ell}(t), p(t) - p^*
    \rinner
    +
    \frac{K}{T}
    \sum_{t=1}^T
    \linner
    \ell(t), \frac{1}{K}\mathbf{1} - \chi_{i^*}
    \rinner
    \right]
    \leq
    \E \left[
    \sum_{t=1}^T
    \linner
    \hat{\ell}(t), p(t) - p^*
    \rinner
    \right]
    +
    K .
    \label{eq:boundRiT}
  \end{align}
  From Lemma~\ref{lem:boundOFTRL},
  the term of
  $
    \sum_{t=1}^T \linner
      \hat{\ell}(t), p(t) - p^*
    \rinner
  $
  in \eqref{eq:boundRiT} is bounded as \eqref{eq:boundOFTRL}.
  Penalty terms and stability terms in \eqref{eq:boundOFTRL} are bounded as follows:
  \paragraph{\uline{Penalty terms} in \eqref{eq:boundOFTRL}}
  As the regularizer $\psi_t(p)$ is defined as
  $\sum_{i=1}^K \beta_i(t) \phi(p_i)$ as \eqref{eq:defpsi},
  we have
  \begin{align}
    \psi_{t} (p^*)
    =
    \sum_{i=1}^K \beta_i(t) \phi(p^*_i)
    \leq
    \sum_{i=1}^K \beta_i(t) \max_{x \in [1/T, 1]}  \phi(x)
    \leq
    \sum_{i=1}^K \beta_i(t)
    \max \{ \phi(1/T), \phi(1) \},
    \label{eq:boundpsi}
  \end{align}
  where the first inequality follows from the definition of $p^*$ and
  the second inequality holds since $\phi$ is a convex function.
  Further,
  from the definition of $\phi$ in \eqref{eq:defpsi},
  we have
  \begin{align*}
    \max \{ \phi(1/T), \phi(1) \}
    &
    =
    \max \left\{
      \frac{1}{T} - 1 + \log T + \gamma \left( \frac{1}{T} + \left( 1 - \frac{1}{T} \right) \log \left( 1 - \frac{1}{T} \right)  \right),
      \gamma
    \right\}
    \\
    &
    \leq
    \max \left\{
      \frac{1+\gamma}{T} - 1 + \log T ,
      \gamma
    \right\}
    = \gamma,
  \end{align*}
  where the last inequality follows from $\gamma = \log T$.
  From this and \eqref{eq:boundpsi},
  we have
  \begin{align}
    \label{eq:boundpsiT1}
    \psi_{T+1} (p^*) \leq \gamma \sum_{i=1}^K \beta_i(T+1).
  \end{align}
  Further,
  as we have
  $\beta_i(t) \leq \beta_i(t+1)$ from \eqref{eq:defalphabeta} and $\phi(x) \geq 0$ for any $x \in (0, 1]$,
  we have
  \begin{align}
    \nonumber
    &
    - \psi_1(q(1))
    +
    \sum_{t=1}^T
    (\psi_{t}(q(t+1))
    - \psi_{t+1}(q(t+1))
    )
    \\
    &
    =
    -
    \sum_{i=1}^K
    \left(
      \beta_i(1) \phi(q_i(1) +
      \sum_{t=1}^{T}(\beta_i(t+1) - \beta_i(t))\phi(q_i(t+1))
    \right)
    \leq 0.
    \label{eq:boundsumpsi}
  \end{align}
  Combining \eqref{eq:boundpsiT1} and \eqref{eq:boundsumpsi},
  we obtain
  \begin{align}
    \psi_{T+1}(p^*)
    - \psi_1(q(1))
    +
    \sum_{t=1}^T
    (\psi_{t}(q(t+1))
    - \psi_{t+1}(q(t+1))
    )
    \leq \gamma \beta_i(T+1).
    \label{eq:boundpenalty}
  \end{align}
  \paragraph{\uline{Stability terms} in \eqref{eq:boundOFTRL}}
  The Bregman divergence $D_t(p, q)$ is expressed as
  \begin{align*}
    \nonumber
    D_t(p, q)
    &
    =
    \sum_{i=1}^K
    \left(
      \beta_i(t) D^{(1)}(p_i, q_i)
      +
      \beta_i(t) \gamma D^{(2)}(p_i, q_i)
    \right)
    \\
    &
    \geq
    \sum_{i=1}^K
    \max \left\{
      \beta_i(t) D^{(1)}(p_i, q_i),
      \beta_i(t) \gamma D^{(2)}(p_i, q_i)
    \right\}
  \end{align*}
  where $D^{(1)}$ and $D^{(2)}$ stand for Bregman divergences associated with $\phi^{(1)}(x) = - \log x$ and
  $\phi^{(2)}(x) = (1-x)\log(1-x)$.
  We hence have
  \begin{align}
    &
    \linner
    \hat{\ell}({t}) - m(t),
    p(t) - q(t+1)
    \rinner
    -
    D_t ( q(t+1) , p(t))
    \nonumber
    \\
    &
    \leq
    \sum_{i=1}^K
    \left(
      (\hat{\ell}_i(t) - m_i(t))(p_{i}(t) - q_{i}(t+1))
      -
      \beta_i(t)
      \max \left\{
         D^{(1)}(q_i(t+1), p_i(t)),
         \gamma D^{(2)}(q_i(t+1), p_i(t))
      \right\}
    \right)
    \nonumber
    \\
    &
    \leq
    \sum_{i=1}^K
    \left(
      \min \left\{
        \beta_i(t)
        g\left(\frac{p_i(t)(\hat{\ell}_i(t) - m_i(t))}{\beta_i(t)}\right),
        \beta_i(t) \gamma (1 - p_{i}(t))
        h \left( \frac{\hat{\ell}_i (t) - m_i(t)}{\gamma \beta_i(t)} \right)
      \right\}
    \right),
    \label{eq:boundip}
  \end{align}
  where the last inequality follows from Lemma~\ref{lem:boundD12}
  and $g$ and $h$ are defined in \eqref{eq:defgh}.
  From this,
  since $\hat{\ell}(t) - m(t)$ is expressed as
  $\hat{\ell}(t) - m(t) = \frac{\ell_{j}(t) - m_j(t)}{p_j(t)} \chi_j$ with $j = I(t)$,
  we have
  \begin{align}
    &
    \linner
    \hat{\ell}({t}) - m(t),
    p(t) - q(t+1)
    \rinner
    -
    D_t ( q(t+1) , p(t))
    \nonumber
    \\
    &
    \leq
      \min \left\{
        \beta_j(t) g \left( \frac{\ell_j(t) - m_j(t)}{\beta_j(t)} \right),
        \beta_j(t) \gamma (1 - p_{j}(t))
        h \left( \frac{\ell_j (t) - m_j(t)}{\gamma \beta_j(t)p_j(t)} \right)
      \right\}
    \nonumber
    \\
    &
    \leq
    \left\{
      \begin{array}{ll}
        \frac{(\ell_j(t) - m_j(t))^2 }{2\beta_j(t) }
        + \frac{\delta|\ell_j(t) - m_j(t)|^3 }{\beta_j(t)^2}
        &
        (\gamma p_j(t) \leq 1 )
        \\
        \min\left\{
        \frac{(\ell_j(t) - m_j(t))^2 }{2\beta_j(t)}
        + \frac{\delta|\ell_j(t) - m_j(t)|^3 }{\beta_j(t)^2},
        \frac{(1-p_j(t))(\ell_j(t) - m_j(t))^2 }{\gamma p_j(t)^2 \beta_j(t)}
        \right\}
        &
        (\gamma p_j(t) > 1 )
      \end{array}
    \right.
    \nonumber
    \\
    &
    \leq
    \min\left\{
    \frac{(\ell_j(t) - m_j(t))^2 }{2\beta_j(t)}
    + \frac{\delta|\ell_j(t) - m_j(t)|^3 }{\beta_j(t)^2},
    \frac{(1-p_j(t))(\ell_j(t) - m_j(t))^2 }{\gamma p_j(t)^2 \beta_j(t)}
    \right\}
    \nonumber
    \\
    &
    \leq
    \left(
    \frac{1}{2 \beta_j(t)}
    +
    \frac{\delta}{\beta_j(t)^2}
    \right)
    {(\ell_j(t) - m_j(t))^2}
    \min\left\{
      1,
      \frac{2(1-p_j(t))}{\gamma p_j(t)^2}
    \right\}
    =
    \sum_{i=1}^K
    \left(
    \frac{1}{2 \beta_i(t)}
    +
    \frac{\delta}{\beta_i(t)^2}
    \right)\alpha_i(t),
    \label{eq:boundstab}
  \end{align}
  where the first inequality follows from \eqref{eq:defellhat} and \eqref{eq:boundip},
  the second inequality follows from \eqref{eq:boundg}, \eqref{eq:boundhh} and the fact that $|\frac{\ell_j(t) - m_j(t)}{\beta_j(t)}| \leq \frac{1}{\bet} \leq 1$,
  and the third inequality holds since $\gamma p_j(t) \leq 1$ means $\frac{1-p_j(t)}{\gamma p_j(t)^2} \geq \frac{1 - 1/\gamma}{\gamma (1/\gamma)^2} = \gamma - 1 \geq \frac{1}{2} + \delta$,
  which implies
  $
    \frac{(\ell_j(t) - m_j(t))^2 }{2\beta_j(t)}
    + \frac{\delta |\ell_j(t) - m_j(t)|^3 }{\beta_j(t)^2} \leq
    \frac{(1-p_j(t))(\ell_j(t) - m_j(t))^2 }{\gamma p_j(t)^2 \beta_j(t)}
  $.
  We hence have
  \begin{align*}
    \sum_{t=1}^T
    \left(
    \linner
    \hat{\ell}({t}) - m(t),
    p(t) - q(t+1)
    \rinner
    -
    D_t ( q(t+1) , p(t))
    \right)
    \leq
    \sum_{i=1}^K
    \sum_{t=1}^T
    \left(
    \frac{1}{2 \beta_i(t)}
    +
    \frac{\delta}{\beta_i(t)^2}
    \right)\alpha_i(t).
  \end{align*}
  We further have
  \begin{align}
    \label{eq:boundsumalphabeta}
    \sum_{t=1}^T \frac{\alpha_i(t)}{2 \beta_i(t)}
    \leq
    \gamma
    \left(
    \sqrt{ \bet^2 - \frac{1}{\gamma} + \frac{1}{\gamma} \sum_{t=1}^T \alpha_i(t)  }
    -
    \sqrt{\bet^2 - \frac{1}{\gamma}}
    \right)
    \leq
    \gamma
    \left(
    \beta_{i}(T+1)
    -
    \bet
    \right)
    + 1 .
  \end{align}
  The first inequality in \eqref{eq:boundsumalphabeta} can be confirmed via the following:
  \begin{align*}
    &
    \sqrt{ \bet^2 - \frac{1}{\gamma} + \frac{1}{\gamma} \sum_{s=1}^{t} \alpha_i(s)  }
    -
    \sqrt{ \bet^2 - \frac{1}{\gamma} + \frac{1}{\gamma} \sum_{s=1}^{t-1} \alpha_i(s)  }
    \\
    &
    =
    \frac{\alpha_i(t)}
    {
      \gamma \left(
      \sqrt{ \bet^2 - \frac{1}{\gamma} + \frac{1}{\gamma} \sum_{s=1}^{t} \alpha_i(s)  }
      +
      \sqrt{ \bet^2 - \frac{1}{\gamma} + \frac{1}{\gamma} \sum_{s=1}^{t-1} \alpha_i(s)  }
      \right)
    }
    \geq
    \frac{\alpha_i(t)}
    {
      2 \gamma
      \sqrt{ \bet^2 + \frac{1}{\gamma} \sum_{s=1}^{t-1} \alpha_i(s)  }
    }
    =
    \frac{\alpha_i(t)}
    {
      2 \gamma \beta_i(t)
    } ,
  \end{align*}
  where we used $\alpha_i(t) \leq 1$.
  The second inequality in \eqref{eq:boundsumalphabeta} follows from
  \begin{align*}
    \beta_{0} - \sqrt{\bet^2 - \frac{1}{\gamma}}
    =
    \frac{ 1 }{\gamma}
    \frac{1}{
    \beta_{0} + \sqrt{\bet^2 - \frac{1}{\gamma}}
    }
    \leq
    \frac{ 1 }{\gamma}.
  \end{align*}
  Similarly,
  we have
  \begin{align}
    \nonumber
    \sum_{t=1}^T
    \frac{\alpha_i(t)}{\beta_i(t)^2}
    &
    =
    \sum_{t=1}^T
    \frac{\alpha_i(t)}{\bet^2 + \frac{1}{\gamma} \sum_{s=1}^{t-1} \alpha_i(s) }
    =
    \gamma
    \sum_{t=1}^T
    \frac{\alpha_i(t)}{\gamma \bet^2 + \sum_{s=1}^{t-1} \alpha_i(s) }
    \\
    &
    \leq
    \gamma
    \log \left(
      1 + \frac{1}{\gamma {\beta}_0^2 - 1 } \sum_{t=1}^T \alpha_i(t)
    \right)
    \leq
    2 \gamma
    \log
    \frac{\beta_{i}(T+1)}{\beta_i(1)}
    +
    2
    .
    \label{eq:boundalphabeta2}
  \end{align}
  The first inequality in \eqref{eq:boundalphabeta2} can be shown via the following:
  \begin{align*}
    &
    \log \left(
      1 + \frac{1}{\gamma {\beta}_0^2 - 1 } \sum_{s=1}^t \alpha_i(s)
    \right)
    -
    \log \left(
      1 + \frac{1}{\gamma {\beta}_0^2 - 1 } \sum_{s=1}^{t-1} \alpha_i(s)
    \right)
    \\
    &
    =
    -
    \log
    \left(
      1
      -
      \frac{
        \alpha_i(t)
      }
      {
      \gamma {\beta}_0^2 - 1 + \sum_{s=1}^t \alpha_i(s)
      }
    \right)
    \geq
    \frac{
      \alpha_i(t)
    }
    {
    \gamma {\beta}_0^2 + \sum_{s=1}^{t-1} \alpha_i(s)
    },
  \end{align*}
  where the last inequality follows from
  $\alpha_i(t) \leq 1$.
  The second inequality in \eqref{eq:boundalphabeta2} follows from
  \begin{align*}
    &
    \log \left(
      1 + \frac{1}{\gamma {\bet}^2 - 1 } \sum_{t=1}^T \alpha_i(t)
    \right)
    \leq
    \log \left(
      1 + \frac{1}{\gamma {\bet}^2} \sum_{t=1}^T \alpha_i(t)
    \right)
    +
    \log \frac{\gamma{\bet}^2}{\gamma \bet^2 - 1}
    \\
    &
    =
    \log \left(
      \frac{\beta_{i}(T+1)^2}{\bet^2}
    \right)
    +
    \log \left(
      1
      +
      \frac{1}{\gamma \bet^2 - 1}
    \right)
    \leq
    2
    \log
      \frac{\beta_{i}(T+1)}{\bet}
    +
    \frac{2}{\gamma} .
  \end{align*}
  Combining \eqref{eq:boundstab}, \eqref{eq:boundsumalphabeta}, and \eqref{eq:boundalphabeta2},
  we obtain
  \begin{align*}
    &
    \sum_{t=1}^T
    \left(
    \linner
    \hat{\ell}({t}) - m(t),
    p(t) - q(t+1)
    \rinner
    -
    D_t ( q(t+1) , p(t))
    \right)
    \\
    &
    \leq
    \gamma
    \sum_{i=1}^K
    \left(
      \beta_i(T+1)
      -
      {\beta}_i(1)
      +
      2
      \delta
      \log \frac{\beta_i(T+1)}{\beta_i(1)}
    \right)
    +
    K(1 + 2 \delta ).
  \end{align*}
  Combining this,
  \eqref{eq:boundRiT},
  \eqref{eq:boundOFTRL} and \eqref{eq:boundpenalty},
  we obtain
  \begin{align*}
    R_{i^*}(T)
    \leq
    \gamma
    \sum_{i=1}^K
    \E \left[
      2
      \beta_i(T+1)
      -
      {\beta}_i(1)
      +
      2
      \delta \log
      \frac{\beta_i(T+1)}{\beta_i(1)}
    \right]
    +
    2 K (1 + \delta).
  \end{align*}
\qed


\subsection{Proof of Lemma~\ref{lem:bound-sumalpha-suboptimal}}
\begin{lemma}
  \label{lem:logTregret}
  Suppose $\ell(s) \in [0,1]$ for any $s \in [t]$ and
  define $m(t) \in [0, 1]$ by
  \begin{align}
    m(t) = \frac{1}{ t }\left(
      \frac{1}{2} + \sum_{s=1}^{t-1} \ell(s)
    \right).
  \end{align}
  We then have
  \begin{align}
    \sum_{t=1}^T ((\ell(t)-m(t))^2 - (\ell(t)-m^*)^2) \leq \frac{5}{4} + \log T
  \end{align}
  for any $m^* \in [0, 1]$.
\end{lemma}
\begin{proof}
  From the definition of $m(t)$,
  $m(t)$ is expressed as
  \begin{align*}
    m(t) \in
    \argmin_{m \in \re} \left\{
      \left( m - \frac{1}{2} \right)^2 + \sum_{s=1}^{t-1} \left( m - \ell(s) \right)^2
    \right\} .
  \end{align*}
  From this,
  as the function
  $ \left( m - \frac{1}{2} \right)^2 + \sum_{s=1}^{t-1} \left( m - \ell(s) \right)^2 $
  is a quadratic function in $m$ with the leading coefficient $t$,
  we have
  \begin{align}
    \label{eq:quadraticopt}
    \left( m - \frac{1}{2} \right)^2 + \sum_{s=1}^{t-1} \left( m - \ell(s) \right)^2
    =
    \left( m(t) - \frac{1}{2} \right)^2 + \sum_{s=1}^{t-1} \left( m(t) - \ell(s) \right)^2
    +
    t (m - m(t))^2
  \end{align}
  for any $m \in \re$.
  Using this,
  we obtain
  \begin{align*}
    &
    \sum_{t=1}^T (\ell(t) - m^*)^2 + \left( m^* - \frac{1}{2}  \right)^2
    \\
    &
    =
    \sum_{t=1}^T (\ell(t) - m(T+1))^2 + \left( m(T+1) - \frac{1}{2}  \right)^2
    +
    (T+1) (m^* - m(T+1))^2
    \\
    &
    \geq
    \sum_{t=1}^T (\ell(t) - m(T+1))^2 + \left( m(T+1) - \frac{1}{2}  \right)^2
    \\
    &
    =
    \sum_{t=1}^{T-1} (\ell(t) - m(T+1))^2 + \left( m(T+1) - \frac{1}{2}  \right)^2
    +
    (\ell(T) - m(T+1))^2
    \\&
    =
    \sum_{t=1}^{T-1} (\ell(t) - m(T))^2 + \left( m(T) - \frac{1}{2}  \right)^2
    +
    (\ell(T) - m(T+1))^2
    +
    T (m(T+1) - m(T))^2
    \\
    &
    =
    \left(m(1) - \frac{1}{2}\right)^2
    +
    \sum_{t=1}^T (\ell(t) - m(t+1))^2
    +
    \sum_{t=1}^T t (m(t+1) - m(t))^2
    \\
    &
    =
    \sum_{t=1}^T (\ell(t) - m(t+1))^2
    +
    \sum_{t=1}^T t (m(t+1) - m(t))^2,
  \end{align*}
  where the first and the third equality follow from \eqref{eq:quadraticopt},
  and the forth equality can be shown by repeating the same transformation $T$ times.
  We hence have
  \begin{align*}
    &
    \sum_{t=1}^T \left((\ell(t)-m(t))^2 - (\ell(t)-m^*)^2 \right)
    \\
    &
    \leq
    \sum_{t=1}^T \left( (\ell(t)-m(t))^2 - (\ell(t)-m(t+1))^2 - t(m(t+1)-m(t))^2 \right) + \left( m^* - \frac{1}{2} \right)^2
    \\
    &
    =
    \sum_{t=1}^T \left((2\ell(t)-m(t) - m(t+1))(m(t+1) - m(t) ) - t(m(t+1)-m(t))^2 \right) + \left( m^* - \frac{1}{2} \right)^2
    \\
    &
    \leq
    \sum_{t=1}^T \frac{1}{4t}(2\ell(t)-m(t) - m(t+1))^2
    +
    \left( m^* - \frac{1}{2} \right)^2
    \\
    &
    \leq
    \sum_{t=1}^T \frac{1}{t}
    +
    \left( m^* - \frac{1}{2} \right)^2
    \leq
    \frac{5}{4}
    +
    \log T,
  \end{align*}
  where the first inequality follows from
  $a x - t x^2 = \frac{a^2}{4t} - ( \frac{a}{2 \sqrt{t}} - \sqrt{t}x )^2 \leq \frac{a^2}{4t}$ that holds for any $a, x \in \re$,
  and the last inequality holds since
  $ |2\ell(t)-m(t) - m(t+1) | \leq 2$ follows from $\ell(t), m(t) \in [0, 1]$.
\end{proof}

\quad\\
\textbf{Proof of Lemma~\ref{lem:bound-sumalpha-suboptimal}}\quad
  From the definition of $\alpha_i(t)$,
  we have
  \begin{align*}
    \sum_{t=1}^T \alpha_i(t)
    &
    \leq
    \sum_{t=1}^T \mathbf{1}[I(t) = i] (\ell_i(t) - m_i(t))^2
    \\
    &
    \leq
    \sum_{t=1}^T \mathbf{1}[I(t) = i] (\ell_i(t) - m^*_i)^2
    +
    \frac{5}{4}
    +
    \log (1 + \sum_{t=1}^T \mathbf{1}[I(t) = i] )
    \\
    &
    =
    \sum_{t=1}^T \mathbf{1}[I(t) = i] (\ell_i(t) - m^*_i)^2
    +
    \frac{5}{4}
    +
    \log (1 + N_i(T)),
  \end{align*}
  where the inequality follows from Lemma~\ref{lem:logTregret} and the definition of $m_i(t)$ given in \eqref{eq:defmi}.
\qed

\subsection{Proof of Lemma~\ref{lem:bound-sumalpha-optarm}}
\begin{proof}
  From the definition of $\alpha_i(t)$ in \eqref{eq:defalphabeta},
  we have
  \begin{align*}
    \E[\alpha_i(t) | p_i(t)]
    &
    =
    \E\left[
      \mathbf{1}[I(t) = i]
      (\ell_i(t) - m_i(t))^2
      \min\left\{
        1,
        \frac{2(1-p_i(t))}{\gamma p_i(t)^2 }
      \right\}
      |
      p_i(t)
    \right]
    \\
    &
    \leq
    \E\left[
      \mathbf{1}[I(t) = i]
      \min\left\{
        1,
        \frac{2(1-p_i(t))}{\gamma p_i(t)^2 }
      \right\}
      |
      p_i(t)
    \right]
    \\
    &
    =
      \min\left\{
        p_i(t),
        \frac{2(1-p_i(t))}{\gamma p_i(t) }
      \right\}
    \leq
    \left\{
      \begin{array}{ll}
        p_i(t) & (p_i(t) < \frac{1}{\sqrt{\gamma}} ) \\
        2 \frac{1-p_i(t)}{ \sqrt{\gamma}} & (p_i(t) \geq \frac{1}{\sqrt{\gamma}} )
      \end{array}
    \right.
    \leq
    \frac{2}{\sqrt{\gamma}}
      (
      1 - p_i(t)
      ),
  \end{align*}
  where the first inequality follows from the condition of $\ell_i(t), m_i(t) \in [0, 1]$
  and
  the last inequality is due to $\sqrt{\gamma} \geq 2$ that follows from the assumption of $T \geq 55$.
\end{proof}

\subsection{Proof of the regret bound of (\ref{eq:thm1-bound-sto}) in stochastic settings}
\label{sec:appendix-sto}
This subsection shows
the regret bound in \eqref{eq:thm1-bound-sto} by using Lemmas~\ref{lem:boundRTbeta},
\ref{lem:bound-sumalpha-suboptimal}, and \ref{lem:bound-sumalpha-optarm}.
More precisely,
we will show the following:
\begin{proposition}
  \label{prop:bound-sto}
  In any stochastic setting with a unique optimal arm,
  the proposed algorithm achieves
  \begin{align}
    \label{eq:prop-bound-sto}
    \limsup_{T \to \infty}
    \frac{R(T)}{\log T}
    \leq
    \sum_{i \in [K] \setminus \{ i^* \}}
    h \left(
      \frac{\sigma_i^2}{\Delta_i}
    \right)
    +
    2 \bet
  \end{align}
  where $h: \re_{\geq 0} \to \re$ is defined as
  \begin{align}
    h(z)
    &
    =
    \left\{
      \begin{array}{l}
        2 \bet
        \quad  \mbox{if} \quad  0 \leq z \leq \frac{\bet}{2(1 + \delta/\bet )}  ,
        \\
        \\
        2 z \left({ 1 + \sqrt{1 + 2 \frac{\delta }{z} }}\right) - 2 \delta
        +
        4 \delta \left(
          \log \frac{z}{\bet}
          +
          \log \left(
            { 1 + \sqrt{1 + 2 \frac{\delta }{z} }}
          \right)
        \right)
        +
        \frac{\bet^2}{z}
        -
        2\bet
        \\
        \quad \quad \quad \quad \quad \quad \quad \quad \quad \quad \quad \quad \quad \quad
        \quad \quad \quad \quad \quad \quad \quad \quad \quad \quad
          \mbox{if} \quad z > \frac{\bet}{2(1 + \delta/\bet )}
      \end{array}
    \right. .
    \label{eq:defhh}
  \end{align}
\end{proposition}
This proposition leads to \eqref{eq:thm1-bound-sto} in Theorem~\ref{thm:main}.
In fact,
for
$
  z > \frac{\bet}{2(1 + \delta/\bet )}
$,
$h(z)$ in \eqref{eq:defhh} is bounded as
\begin{align*}
  h (z)
  &
  \leq
  2 z \left( 1 + 1 + \frac{\delta }{z} \right) - 2 \delta
  +
  4 \delta \left(
    \log z
    +
    \log \left(
      { 1 + \sqrt{1 + 2 \frac{\delta }{z} }}
    \right)
  \right)
  +
  \frac{\bet^2}{\bet}
  \cdot
  2 \left(1 +  \frac{\delta}{\bet}\right)
  -
  2\bet
  \\
  &
  =
  4 z
  +
  4 \delta
  \left(
    \log z
    +
    \log \left(
      { 1 + \sqrt{1 + 2 \frac{\delta }{z} }}
    \right)
    +
    \frac12
  \right)
  \\
  &
  \leq
  4z + c \log (1 + z)
  \quad
  \left(
    c = O\left(
      \delta^2
    \right)
    =
    O\left(
      \left( \log \frac{1}{\epsilon} \right)^2
    \right)
  \right),
\end{align*}
where the last inequality follows from $\log (1+z) = \Omega( 1 / \delta )$
that holds for
$
  z > \frac{\bet}{2(1 + \delta/\bet )}
$.
Hence,
for any $z \geq 0$,
$h(z)$ is bounded as
\begin{align}
  h(z) \leq \max \left\{
    4 z + c \log (1+z),
    2 \bet
  \right\} .
  \label{eq:boundh}
\end{align}
From this and Proposition~\ref{prop:bound-sto},
recalling that $\bet = 1 + \epsilon$,
we obtain \eqref{eq:thm1-bound-sto} in Theorem~\ref{thm:main}.
Further,
in the case of $\epsilon = 0.2$,
we can confirm numerically that
$h(z) \leq \max\left\{ 4 z + 4.2 \log (1+z), 2.4 \right\}$,
which leads to the regret bound described in Table~\ref{table:regretbound-sto}.

The rest of this subsection will provide proof of Proposition~\ref{prop:bound-sto}.
As a preliminary,
we show the following lemma.
\begin{lemma}
  \label{lem:maxabc}
  It holds for any $a,b,c>0$ that
  \begin{align}
    \max_{y \geq 0}
    \left\{
      a y
      +
      b \log y
      -
      c y^2
    \right\}
    =
    \frac{1}{2}
    \left(
      \frac{a}{4c} \left({a + \sqrt{a^2 + 8bc}}\right) - b
    \right)
    +
    b
    \log \frac{a + \sqrt{a^2 + 8bc}}{4c}  .
    \label{eq:maxabc}
  \end{align}
\end{lemma}
\begin{proof}
  Denote
  $f(y)
    =
    a y
    +
    b \log y
    -
    c y^2
  $
  and let $f'$ represent the derivative of $f$.
  As $f$ is a concave function,
  the maximum is attained by $y^*>0$ satisfying $f'(y^*) = 0$,
  which is rewritten as follows:
  \begin{align*}
    f'(y^*)
    =
    a + \frac{b}{y^*} - 2 c y^* = 0.
  \end{align*}
  This is equivalent to
  \begin{align*}
    2 c y^{*2} - ay^* - b = 0.
  \end{align*}
  By solving this quadratic equation,
  we obtain
  \begin{align*}
    y^* = \frac{a + \sqrt{a^2 + 8 bc}}{4c}.
  \end{align*}
  Hence,
  the maximum of $f$ is expressed as
  \begin{align*}
  \max_{y > 0} f(y)
  &
  =
  f(y^*)
  =
  a y^* + b \log y^*
  - c y^{*2}
  =
  a y^* + b \log y^*
  - \frac{1}{2}(ay^* + b)
  \\
  &
  =
  \frac{1}{2}(ay^* - b) + b \log y^*
  =
  \frac{1}{2}\left(\frac{a}{4c} \left(a + \sqrt{a^2 + 8 bc}\right) - b\right) + b \log \frac{a + \sqrt{a^2 + 8 bc}}{4c},
  \end{align*}
  which completes the proof.
\end{proof}
We are now ready to show Proposition~\ref{prop:bound-sto}.
\quad\\
\textbf{Proof of Proposition~\ref{prop:bound-sto}}\quad
  From Lemma~\ref{lem:bound-sumalpha-suboptimal}
  we have \eqref{eq:bound-sumalpha-suboptimal}.
  Hence,
  for each suboptimal arm $i \in [K] \setminus \{ i^* \}$
  the term in Lemma~\ref{lem:boundRTbeta} can be bounded as follows:
  \begin{align}
    \nonumber
    &
    \E \left[
      2 \beta_i(T+1) - \beta_i(1) + 2 \delta \log \frac{\beta_i(T+1)}{\beta_i(1)}
    \right]
    \\
    \nonumber
    &
    =
    \E \left[
    2 \sqrt{\bet^2 + \frac{1}{\gamma} \sum_{t=1}^T \alpha_{i}(t)}
    -
    \bet
    +
    \delta \log \left(
      1 + \frac{1}{\gamma \bet^2} \sum_{t=1}^T \alpha_{i}(t)
    \right)
    \right]
    \\
    \nonumber
    &
    \leq
    2 \sqrt{\bet^2 + \frac{1}{\gamma} \left( \sigma_i^2 P_i + \log(1 + P_i) + \frac{5}{4} \right)}
    -
    \bet
    +
    \delta \log \left(
      1 + \frac{1}{\gamma \bet^2}
      \left( \sigma_i^2 P_i + \log(1 + P_i) + \frac{5}{4} \right)
    \right)
    \\
    \nonumber
    &
    \leq
    2 \sqrt{\bet^2 + \frac{\sigma_i^2 P_i}{\gamma}}
    +
    \frac{1}{\gamma \bet}
    \left( \log(1 + P_i) + \frac{5}{4} \right)
    -
    \bet
    +
    \delta \log \left(
      1 + \frac{\sigma_i^2 P_i}{\gamma \bet^2}
    \right)
    +
    \frac{\delta}{\gamma \bet^2}
    \left( \log(1 + P_i) + \frac{5}{4} \right)
    \\
    &
    \leq
    2 \sqrt{ \bet^2 + \frac{\sigma_i^2 P_i}{\gamma}}
    -
    \bet
    +
    \delta \log \left(
      1 + \frac{\sigma_i^2 P_i}{\gamma \bet^2}
    \right)
    +
    \frac{\xi}{\gamma}
    \left( \log(1 + P_i) + \frac{5}{4} \right) ,
    \label{eq:bound-beta-suboptimal}
  \end{align}
  where the first inequality follows from \eqref{eq:bound-sumalpha-suboptimal},
  the second inequality follows from $\sqrt{x + y} \leq \sqrt{x} + \frac{y}{2 \sqrt{x}}$ that holds for any $x > 0$ and $y \geq 0$,
  and $\log(1 + x + y) \leq \log(1 + x) + y$ that holds for any $x, y \geq 0$.
  We denote $\xi = \frac{1}{1+\epsilon} + \frac{\delta}{(1+\epsilon)^2}$.

  From Lemma~\ref{lem:bound-sumalpha-optarm},
  for the optimal arm $i^*$,
  we have
  \begin{align}
    &
    \E \left[
      2 \beta_{i^*}(T+1) - \beta_{i^*}(1) + 2 \delta \log \frac{\beta_{i^*}(T+1)}{\beta_{i^*}(1)}
    \right]
    \nonumber
    \\
    &
    =
    \E \left[
    2 \sqrt{\bet^2 + \frac{1}{\gamma} \sum_{t=1}^T \alpha_{i^*}(t)}
    -
    \bet
    +
    \delta \log \left(
      1 + \frac{1}{\gamma \bet^2} \sum_{t=1}^T \alpha_{i^*}(t)
    \right)
    \right]
    \nonumber
    \\
    &
    \leq
    \E
    \left[
    2 \sqrt{ \bet^2 + \frac{1}{\gamma} \sum_{t=1}^T \alpha_{i^*}(t)}
    -
    \bet
    +
    2 \delta
    \left(
      \sqrt{
        1+
        \frac{1}{\gamma \bet^2} \sum_{t=1}^T \alpha_{i^*}(t)
      }
      - 1
    \right)
    \right]
    \nonumber
    \\
    &
    =
    2
    \left(
      \bet + \delta
    \right)
    \E
    \left[
      \sqrt{
        1+
        \frac{1}{\gamma \bet^2 } \sum_{t=1}^T \alpha_{i^*}(t)
      }
      - 1
    \right]
    +
    \bet
    \nonumber
    \\
    &
    \leq
    2
    \left(
      \bet + \delta
    \right)
    \left(
      \sqrt{
        1+
        \frac{2}{\gamma^{3/2} \bet^2} \sum_{i \in [K] \setminus \{ i^* \}} P_i
      }
      - 1
    \right)
    +
    \bet
    \nonumber
    \\
    &
    \leq
    2
    \left(
      1 + \delta
    \right)
    \sum_{i \in [K] \setminus \{ i^* \}}
      \sqrt{
        \frac{2}{\gamma^{3/2} } P_i
      }
    +
    \bet,
    \label{eq:bound-beta-optarm}
  \end{align}
  where the first inequality follows from the inequality of $\log(1+x) \leq 2 (\sqrt{1+x} - 1)$
  and the second inequality follows from Lemma~\ref{lem:bound-sumalpha-optarm}.
  Combining Lemma~\ref{lem:boundRTbeta} with \eqref{eq:bound-beta-suboptimal} and \eqref{eq:bound-beta-optarm},
  we obtain
  \begin{align}
    \frac{R_{i^*}(T)}{\gamma}
    &
    \leq
    \sum_{i \in [K] \setminus \{ i^* \}}
    \left(
      2 \sqrt{ \bet^2 + \frac{\sigma_i^2 P_i}{\gamma}}
      -
      \bet
      +
      \delta \log \left(
        1 + \frac{\sigma_i^2 P_i}{\gamma \bet^2}
      \right)
      +
      \frac{\xi}{\gamma}
      \left( \log(1 + P_i) + \frac{5}{4} \right)
    \right)
    \nonumber
    \\
    &
    \quad
    +
    2
    \left(
      1 + \delta
    \right)
    \sum_{i \in [K] \setminus \{ i^* \}}
      \sqrt{
        \frac{2}{\gamma^{3/2} } P_i
      }
    +
    \bet
    +
    \frac{2(1 + \delta)}{\gamma}
    \nonumber
    \\
    &
    =
    \sum_{i \in [K] \setminus \{i^* \}} f_i \left(
      \frac{P_i}{\gamma}
    \right)
    + \bet
    +
    \frac{1}{\gamma}
    \left( 2 + 2 \delta + \frac{5}{4} \xi \right),
    \label{eq:boundRTsto}
  \end{align}
  where we define
  \begin{align}
    \label{eq:deffi}
    f_i(x)
    =
    2 \sqrt{ \bet^2 + \sigma_i^2 x }
    +
    2 ( 1 + \delta ) \sqrt{ \frac{2 x }{ \sqrt{\gamma}  }}
    +
    \delta \log \left( 1 + \frac{\sigma_i^2 x }{\bet^2}   \right)
    +
    \frac{\xi}{\gamma} \log ( 1 + \gamma x )
    -
    \bet
  \end{align}
  On the other hand,
  in stochastic settings,
  the regret can be expressed as
  \begin{align}
    R_{i^*}(T)
    &
    =
    \E
    \left[
      \sum_{t=1}^T ( \ell_{I(t)}(t) - \ell_{i^*}(t))
    \right]
    =
    \E
    \left[
      \sum_{t=1}^T ( \mu_{I(t)} - \mu_{i^*})
    \right]
    \nonumber
    \\
    &
    =
    \E
    \left[
      \sum_{t=1}^T \Delta_{I(t)}
    \right]
    =
    \E
    \left[
      \sum_{t=1}^T \sum_{i=1}^K \Delta_{i} p_i(t)
    \right]
    =
    \sum_{i \in [K] \setminus \{ i^* \}}
    \Delta_i
    \E
    \left[
      \sum_{t=1}^T p_i(t)
    \right]
    =
    \sum_{i \in [K] \setminus \{ i^* \}}
    \Delta_i P_i .
    \label{eq:RTDelta}
  \end{align}
  Combining \eqref{eq:boundRTsto} and \eqref{eq:RTDelta},
  we obtain
  \begin{align}
    &
    \frac{ R_{i^*}(T) }{\log T}
    =
    \frac{ R_{i^*}(T) }{\gamma}
    =
    2\frac{R_{i^*}(T)}{\gamma}
    -
    \frac{R_{i^*}(T)}{\gamma}
    =
    2\frac{R_{i^*}(T)}{\gamma}
    -
    \frac{1}{\gamma}
    \sum_{i \in [K] \setminus \{ i^* \}}
    \Delta_i P_i
    \nonumber
    \\
    &
    \leq
    \sum_{i \in [K] \setminus \{ i^* \}}
    \left(
      2 f_i \left(
        \frac{P_i}{\gamma}
      \right)
      -
      \Delta_i
      \frac{P_i}{\gamma}
    \right)
    -
    2(K-2)\bet
    +
    \frac{2}{\gamma}
    \left( 2 + 2 \delta + \frac{5}{4} \xi \right)
    \nonumber
    \\
    &
    \leq
    \sum_{i \in [K] \setminus \{ i^* \}}
    \max_{x \geq 0}
    \left\{
      2 f_i (x)
      -
      \Delta_i x
    \right\}
    +
    2\bet
    +
    \frac{2}{\gamma}
    \left( 2 + 2 \delta + \frac{5}{4} \xi \right).
    \label{eq:limRTgamma}
  \end{align}
  Let us evaluate $
    \max_{x \geq 0}
    \left\{
      2 f_i (x)
      -
      \Delta_i x
    \right\}
  $.
  The derivative of $f_i(x)$ is expressed as
  \begin{align*}
    f'_i(x)
    :=
    \frac{\mathrm{d}f_i}{\mathrm{d}x}(x)
    =
    \frac{\sigma_i^2}{\sqrt{ \bet^2 + \sigma_i^2 x}}
    +
    \frac{2
      \left(1+{\delta} \right)
    }{\gamma^{1/4}}
    \frac{1}{\sqrt{ 2 x} }
    +
    \frac{\delta \sigma_i^2 }{\bet^2 + \sigma_{i}^2 x}
    +
    \frac{\xi}{1 + \gamma x} .
  \end{align*}
  As $f_i$ is a concave function,
  the maximum of $ 2 f_i(x) - \Delta_i x $ is attained by $x_i^* \in \re$ satisfying $2 f'(x_i^*) = \Delta_i$.
  Define $\tilde{x}_i \geq 0$ by
  \begin{align}
    \tilde{x}_i
    :=
    \max \left\{
      \left(
      \frac{4 \sigma_i}{\Delta_i}
      \right)^2,
      \frac{8\delta}{\Delta_i},
      \frac{16\xi}{\gamma \Delta_i},
      \frac{2^9 (1+\delta)^2}{\sqrt{\gamma} \Delta_i^2}
    \right\}.
  \end{align}
  We then have
  \begin{align*}
    2f_i'(\tilde{x}_i)
    &
    \leq
    \frac{2 \sigma_i}{\sqrt{
      \left(
      \frac{4 \sigma_i}{\Delta_i}
      \right)^2
    }}
    +
    \frac{4(1+\delta)}{\gamma^{1/4}}
    \frac{1}{\sqrt{2 \cdot
      \frac{2^9 (1+\delta)^2}{\sqrt{\gamma} \Delta_i^2}
    }}
    +
    \frac{2 \delta \sigma_i^2 }{{\bet}^2 + \sigma_{i}^2
      \frac{8\delta}{\Delta_i}
    }
    +
    \frac{ 2 \xi}{1 + \gamma
      \frac{16\xi}{\gamma \Delta_i}
    }
    \\
    &
    \leq
    \frac{\Delta_i}{2}
    + \frac{\Delta_i}{8}
    + \frac{\Delta_i}{4}
    + \frac{\Delta_i}{8}
    =
    \Delta_i,
  \end{align*}
  which implies $\tilde{x}_i \geq x_i^*$.
  Hence,
  we have
  \begin{align}
    &
    \max_{x \geq 0} \left\{ 2f_i(x) - \Delta_i x \right\}
    =
    2 f_i(x_i^*) - \Delta_i x_i^*
    \nonumber
    \\
    &
    =
    4 \sqrt{ \bet^2 + \sigma_i^2 x_i^* }
    +
    4 ( 1 + \delta )  \sqrt{ \frac{2 x_i^* }{ \sqrt{\gamma}  }}
    +
    2 \delta \log \left( 1 + \frac{\sigma_i^2 x_i^* }{\bet^2}   \right)
    +
    2 \frac{\xi}{\gamma} \log ( 1 + \gamma x_i^* )
    - \Delta_i x_i^*
    - 2 \bet
    \nonumber
    \\
    &
    \le
    4 \sqrt{ \bet^2 + \sigma_i^2 x_i^* }
    +
    4 ( 1 + \delta )  \sqrt{ \frac{2 \tilde{x}_i }{ \sqrt{\gamma}  }}
    +
    2 \delta \log \left( 1 + \frac{\sigma_i^2 {x}_i^* }{\bet^2}   \right)
    +
    2 \frac{\xi}{\gamma} \log ( 1 + \gamma \tilde{x}_i )
    - \Delta_i x_i^*
    - 2 \bet
    \nonumber
    \\
    &
    \leq
    \max_{x \geq 0}
    \left\{
      4 \sqrt{ \bet^2 + \sigma_i^2 x }
      +
      2 \delta \log \left( 1 + \frac{\sigma_i^2 x }{\bet^2}   \right)
      - \Delta_i x
    \right\}
    +
    4 ( 1 + \delta )  \sqrt{ \frac{2 \tilde{x}_i }{ \sqrt{\gamma}  }}
    +
    2 \frac{\xi}{\gamma} \log ( 1 + \gamma \tilde{x}_i )
    - 2 \bet
    \nonumber
    \\
    &
    =
    \max_{x \geq 0}
    \left\{
      g_i(x)
      - \Delta_i x
    \right\}
    -
    2 \bet
    +
    O\left(
      \frac{1}{\gamma^{1/4}}
    \right),
    \label{eq:limmax2f}
  \end{align}
  where we define
  \begin{align}
    g_i(x)
    =
    4 \sqrt{ \bet^2 + \sigma_i^2 x }
    +
    2 \delta \log \left( 1 + \frac{\sigma_i^2 x }{\bet^2}   \right)
  \end{align}
  From \eqref{eq:limmax2f} and \eqref{eq:limRTgamma},
  we have
  \tnote{change $\lim$ to limsup and liminf (no local tex env, don't change for safe now)}
  \begin{align}
    \limsup_{T \to \infty}
    \frac{R_{i^*}(T)}{\log T}
    \leq
    \sum_{i \in [K] \setminus \{ i^* \}}
    \left(
    \max_{x \geq 0}
    \left\{
      g_i(x)
      - \Delta_i x
    \right\}
    - 2 \bet
    \right)
    +
    2 \bet .
    \label{eq:boundRTg}
  \end{align}
  In the following,
  we use the notation of $z_i = \frac{\sigma_i^2}{\Delta_i}$.
  As we have
  \begin{align*}
    g'_i(x)
    =
    \frac{2 \sigma_i^2}{\sqrt{\bet^2 + \sigma_i^2 x}}
    +
    \frac{2 \delta \sigma_i^2}{\bet^2 + \sigma_i^2 x}
    \leq
    2
    \sigma_i^2
    \left(
      \frac{1}{\bet}
      +
      \frac{\delta}{\bet^2}
    \right),
  \end{align*}
  if $z_i = \frac{\sigma_i^2}{\Delta_i} \leq \frac{1}{2
    (
      {1}/{\bet}
      +
      {\delta}/{\bet^2}
    )
  } = \frac{\bet}{2(1 + \delta/\bet)}$,
  the maximum of $g_i(x) - \Delta_i x$ is attained by $x = 0$.
  Hence,
  we have
  \begin{align}
    \max_{x \geq 0}
    \left\{
      g_i(x)
      - \Delta_i x
    \right\}
    =
    g_i(0)
    =
    4 \bet
    \quad
    \mbox{if}
    \quad
    z_i := \frac{\sigma_i^2}{\Delta_i} \leq \frac{\bet}{2(1 + \delta/\bet)}.
    \label{eq:boundg0}
  \end{align}
  Otherwise,
  we have
  \begin{align*}
    g_i(x) - \Delta_i x
    &
    =
    4 \bet \sqrt{ 1 + \frac{\sigma_i^2 x}{\bet^2} }
    +
    2 \delta \log \left( 1 + \frac{\sigma_i^2 x }{\bet^2}   \right)
    -
    \frac{\bet^2 \Delta_i}{\sigma_i^2}
    \left( 1 + \frac{\sigma_i^2 x}{\bet^2} \right)
    +
    \frac{\bet^2 }{z_i}
    \\
    &
    =
    4 \bet \sqrt{ 1 + \frac{\sigma_i^2 x}{\bet^2} }
    +
    4 \delta \log \left( \sqrt{ 1 + \frac{\sigma_i^2 x }{\bet^2} } \right)
    -
    \frac{\bet^2 }{z_i}
    \left( \sqrt{ 1 + \frac{\sigma_i^2 x}{\bet^2} } \right)^2
    +
    \frac{\bet^2 }{z_i}.
  \end{align*}
  From this,
  by setting
  $y = \sqrt{ 1 + \frac{\sigma_i^2 x}{\bet^2} } $,
  we obtain
  \begin{align}
    \max_{x \geq 0}
    \left\{
      g_i(x)
      - \Delta_i x
    \right\}
    \leq
    \max_{y \geq 0}
    \left\{
      4 \bet y
      +
      4 \delta \log y
      -
      \frac{\bet^2 }{z_i} y^2
    \right\}
    +
    \frac{\bet^2 }{z_i} .
    \label{eq:boundg1}
  \end{align}
  We here use the following:
  \begin{align*}
    \max_{y \geq 0}
    \left\{
      a y
      +
      b \log y
      -
      c y^2
    \right\}
    =
    \frac{1}{2}
    \left(
      \frac{a}{4c} \left({a + \sqrt{a^2 + 8bc}}\right) - b
    \right)
    +
    b
    \log \frac{a + \sqrt{a^2 + 8bc}}{4c}  ,
  \end{align*}
  which holds for any $a, b, c > 0$.
  We hence have
  \begin{align}
    &
    \max_{y \geq 0}
    \left\{
      4 \bet y
      +
      4 \delta \log y
      -
      \frac{\bet^2 }{z_i} y^2
    \right\}
    \nonumber
    \\
    &
    =
    \frac{1}{2}
    \left(
      \frac{4 \bet z_i}{4\bet^2} \left({ 4 \bet + \sqrt{(4 \bet)^2 + 32 \frac{\delta \bet^2}{z_i} }}\right) - 4 \delta
    \right)
    +
    4 \delta
    \log \frac{4 \bet + \sqrt{(4 \bet)^2 + 32 {\delta \bet^2}/{z_i}}}{4 {\bet^2}/{z_i}}
    \nonumber
    \\
    &
    =
    2
    \left(
      z_i \left({ 1 + \sqrt{1 + 2 \frac{\delta }{z_i} }}\right) - \delta
    \right)
    +
    4 \delta \left(
      \log \frac{z_i}{\bet}
      +
      \log \left(
        { 1 + \sqrt{1 + 2 \frac{\delta }{z_i} }}
      \right)
    \right).
    \label{eq:boundg2}
  \end{align}
  Combining \eqref{eq:limmax2f} with \eqref{eq:boundg0}, \eqref{eq:boundg1} and \eqref{eq:boundg2},
  we obtain
  \begin{align}
    \max_{x \geq 0} \left\{
      2 f_i(x) - \Delta_i x
    \right\}
    \leq
    h \left(
      \frac{\sigma_i^2}{\Delta_i}
    \right)
    +
    O \left(
      \frac{1}{\gamma^{1/4}}
    \right)
    ,
    \label{eq:boundmax2f}
  \end{align}
  where $h: \re_{\geq 0} \to \re$ is defined by \eqref{eq:defhh}.
  From \eqref{eq:limRTgamma} and \eqref{eq:boundmax2f},
  we have \eqref{eq:prop-bound-sto}.
\qed

\subsection{Analysis for stochastic settings with adversarial corruption}
\label{sec:appendix-corrupted}
This subsection shows a regret bound for corrupted the stochastic settings
described as follows:
\begin{proposition}
  \label{prop:bound-corrupted}
  In any corrupted stochastic setting with a unique optimal arm,
  the proposed algorithm achieves
  \begin{align}
    \label{eq:prop-bound-corrupted}
    R(T)
    \leq
    \mathcal{R}
    +
    O \left(
      \sqrt{C \mathcal{R}}
    \right) ,
  \end{align}
  where $\mathcal{R}$ is the RHS of \eqref{eq:thm1-bound-sto}
  and $C$ is the corruption level defined in Section~\ref{sec:setting}.
\end{proposition}
\begin{proof}
  In stochastic settings with adversarial corruption,
  using Lemma~\ref{lem:bound-sumalpha-suboptimal} with $m_i^* = \mu_i$ we have
  \begin{align}
    \E \left[
      \sum_{t=1}^T \alpha_i(t)
    \right]
    &
    \leq
    \E \left[
      \sum_{t=1}^{T}\mathbf{1}[I(t) = i] (\ell_i(t) - \mu_i )^2
    + \log (1 + N_i(T))
    \right]
    + \frac{5}{4}
    \nonumber
    \\
    &
    =
    \E \left[
      \sum_{t=1}^{T}p_i(t) (\ell_i(t) - \ell'_i(t) + \ell'_i(t) - \mu_i )^2
    + \log (1 + N_i(T))
    \right]
    + \frac{5}{4}
    \nonumber
    \\
    &
    =
    \E \left[
      \sum_{t=1}^{T}p_i(t) \left( (\ell_i(t) - \ell'_i(t))^2  + \sigma_i^2 \right)
    + \log (1 + N_i(T))
    \right]
    + \frac{5}{4}
    \nonumber
    \\
    &
    \leq
    \sigma_i^2 P_i + \log (1 + P_i) + \frac{5}{4}
    +
    P'_i,
    \label{eq:bound-sumalpha-corrupted}
  \end{align}
  where we define
  \begin{align}
    \label{eq:defPprime}
    P'_i
    =
    \E \left[
      \sum_{t=1}^{T}p_i(t) (\ell_i(t) - \ell'_i(t))^2
    \right].
  \end{align}
  Hence,
  in a similar argument to that of showing \eqref{eq:bound-beta-suboptimal},
  by using \eqref{eq:bound-sumalpha-corrupted} instead of \eqref{eq:bound-sumalpha-suboptimal},
  we obtain
  \begin{align}
    &
    \E \left[
      2 \beta_i(T+1) - \beta_i(1) + 2 \delta \log \frac{\beta_i(T+1)}{\beta_i(1)}
    \right]
    \nonumber
    \\
    &
    =
    \E \left[
    2 \sqrt{\bet^2 + \frac{1}{\gamma} \sum_{t=1}^T \alpha_{i}(t)}
    -
    \bet
    +
    \delta \log \left(
      1 + \frac{1}{\gamma \bet^2} \sum_{t=1}^T \alpha_{i}(t)
    \right)
    \right]
    \nonumber
    \\
    &
    \leq
    2 \sqrt{ \bet^2 + \frac{\sigma_i^2 P_i}{\gamma}}
    -
    \bet
    +
    \delta \log \left(
      1 + \frac{\sigma_i^2 P_i}{\gamma \bet^2}
    \right)
    +
    \frac{\xi}{\gamma}
    \left( \log(1 + P_i) + \frac{5}{4} \right)
    +
    2 \sqrt{\frac{P'_i}{\gamma}}
    +
    \delta \log \left(
      1 + \frac{P'_i}{\gamma \bet^2}
    \right)
    \nonumber
    \\
    &
    \leq
    2 \sqrt{ \bet^2 + \frac{\sigma_i^2 P_i}{\gamma}}
    -
    \bet
    +
    \delta \log \left(
      1 + \frac{\sigma_i^2 P_i}{\gamma \bet^2}
    \right)
    +
    \frac{\xi}{\gamma}
    \left( \log(1 + P_i) + \frac{5}{4} \right)
    +
    \left(
    2 + \frac{\delta}{\bet}
    \right)
    \sqrt{\frac{P'_i}{\gamma}}.
    \label{eq:bound-beta-suboptimal-corrupted}
  \end{align}
  Combining this with
  Lemma~\ref{lem:boundRTbeta} and \eqref{eq:bound-beta-optarm},
  via a similar argument to that of showing \eqref{eq:boundRTsto},
  we have
  \begin{align}
    \frac{R_{i^*}(T)}{\gamma}
    \leq
    \sum_{i \in [K] \setminus \{i^* \}} f_i \left(
      \frac{P_i}{\gamma}
    \right)
    +
    \bet
    +
    \frac{1}{\gamma}
    \left( 2 + 2 \delta + \frac{5}{4} \xi \right)
    +
    \left(
    2 + \frac{\delta}{\bet}
    \right)
    \sum_{i \in [K] \setminus \{ i^* \}}
    \sqrt{\frac{P'_i}{\gamma}},
    \label{eq:boundRTcorrupted0}
  \end{align}
  where $f_i$ is defined in \eqref{eq:deffi}.
  We further have
  \begin{align}
    &
    \sum_{i \in [K] \setminus \{ i^* \}}
    \sqrt{\frac{P'_i}{\gamma}}
    \leq
    \sqrt{
      \frac{K-1}{\gamma}
      \sum_{i \in [K] \setminus \{ i^* \}} P'_i
    }
    =
    \sqrt{
      \frac{K-1}{\gamma}
      \E \left[
      \sum_{t=1}^T
      \sum_{i \in [K] \setminus \{ i^* \}}
      p_i(t)( \ell_i(t) - \ell'_i(t) )^2
      \right]
    }
    \nonumber
    \\
    &
    \leq
    \sqrt{
      \frac{K-1}{\gamma}
      \E \left[
      \sum_{t=1}^T
      \| \ell_i(t) - \ell'_i(t) \|_{\infty}^2
      \right]
    }
    \leq
    \sqrt{
      \frac{K-1}{\gamma}
      \E \left[
      \sum_{t=1}^T
      \| \ell_i(t) - \ell'_i(t) \|_{\infty}
      \right]
    }
    =
    \sqrt{
      \frac{K-1}{\gamma} C
    },
    \label{eq:boundRTcorrupted1}
  \end{align}
  where the first inequality follows from the Cauchy-Schwarz inequality,
  the first equality follows from the definition of $P'_i$ in \eqref{eq:defPprime},
  the second inequality follows from the fact that $\sum_{i \in [K] \setminus \{ i^* \}} p_i(t) \leq 1$,
  and
  the third inequality follows from the assumption of $\ell(t), \ell'(t) \in [0, 1]^K$.
  Combining \eqref{eq:boundRTcorrupted0} and \eqref{eq:boundRTcorrupted1},
  we obtain
  \begin{align}
    \frac{R_{i^*}(T)}{\gamma}
    \leq
    \sum_{i \in [K] \setminus \{i^* \}} f_i \left(
      \frac{P_i}{\gamma}
    \right)
    +
    \bet
    +
    \frac{1}{\gamma}
    \left( 2 + 2 \delta + \frac{5}{4} \xi \right)
    +
    \left(
    2 + \frac{\delta}{\bet}
    \right)
    \sqrt{
      \frac{K-1}{\gamma} C
    } .
    \label{eq:boundRTcorrupted}
  \end{align}
  Further,
  in corrupted stochastic settings,
  the regret can be bounded as
  \begin{align}
    R_{i^*}(T)
    &
    =
    \E
    \left[
      \sum_{t=1}^T ( \ell_{I(t)}(t) - \ell_{i^*}(t))
    \right]
    \nonumber
    \\
    &
    =
    \E
    \left[
      \sum_{t=1}^T ( \ell'_{I(t)}(t) - \ell'_{i^*}(t))
    \right]
    +
    \E
    \left[
      \sum_{t=1}^T  (\ell_{I(t)}(t) - \ell'_{I(t)}(t) + \ell'_{i^*}(t) - \ell_{i^*}(t))
    \right]
    \nonumber
    \\
    &
    \geq
    \sum_{i \in [K] \setminus \{ i^* \}}
    \Delta_i P_i - 2 C ,
    \label{eq:RTDelta-corrupted}
  \end{align}
  where the last inequality follows from a similar argument to that of showing \eqref{eq:RTDelta} and
  the assumptions of corrupted stochastic settings.
  From \eqref{eq:boundRTcorrupted} and \eqref{eq:RTDelta-corrupted},
  for any $\lambda \in (0, 1 ]$,
  we have
  \begin{align}
    \frac{R_{i^*}(T)}{\log T}
    &
    =
    (1 + \lambda) \frac{R_{i^*}(T)}{\gamma}
    -
    \lambda \frac{R_{i^*}(T)}{\gamma}
    \nonumber
    \\
    &
    \leq
    \sum_{i \in [K] \setminus \{i^*\}}
    \max_{x \geq 0} \left\{
      (1 + \lambda)  f_i(x)  - \lambda \Delta_i x
    \right\}
    +
    2
    \left(
      2 + \frac{\delta}{\bet}
    \right)
    \sqrt{
      \frac{K-1}{\gamma} C
    }
    +
    2 \lambda \frac{C}{\gamma}
    \nonumber
    \\
    &
    \quad
    +
    (1+\lambda)
    \left(
      \bet +
      \frac{1}{\gamma}
      \left(
        2 + 2 \delta + \frac{5}{4}\xi
      \right)
    \right),
    \label{eq:limRTgamma-corrupted}
  \end{align}
  which can be shown in a way similar to how we showed \eqref{eq:limRTgamma}.
  Further,
  we have
  \begin{align}
    &
    \max_{x \geq 0} \left\{
      (1 + \lambda) f_i(x) - \lambda \Delta_i x
    \right\}
    =
    \frac{1+\lambda}{2}
    \max_{x \geq 0} \left\{
      2 f_i(x) - \frac{2 \lambda \Delta_i}{1+\lambda} x
    \right\}
    \nonumber
    \\
    &
    \leq
    \frac{1+\lambda}{2}
    h \left(
      \frac{(1+\lambda)\sigma_i^2}{2 \lambda \Delta_i}
    \right)
    +
    O\left( \frac{1}{\gamma^{1/4}} \right)
    \nonumber
    \\
    &
    \leq
    \max\left\{
    \frac{(1+\lambda)^2}{\lambda}
    \frac{\sigma_i^2}{ \Delta_i}
    +
    c
    \log \left(
      1 +
      \frac{\sigma_i^2}{\lambda \Delta_i}
    \right),
    (1 + \lambda) \bet
    \right\}
    +
    O\left( \frac{1}{\gamma^{1/4}} \right)
    \nonumber
    \\
    &
    \leq
    \max\left\{
      4
    \frac{\sigma_i^2}{ \Delta_i}
    +
    c
    \log \left(
      1 +
      \frac{\sigma_i^2}{\Delta_i}
    \right),
    2
    \bet
    \right\}
    +
    (1 + c)
    \left(
    \frac{1}{\lambda} - 1
    \right)
    \frac{\sigma_i^2}{\Delta_i}
    +
    O\left( \frac{1}{\gamma^{1/4}} \right),
    \label{eq:boundlamfDx}
  \end{align}
  where $h(z)$ is defined as \eqref{eq:defhh},
  the inequality follows from \eqref{eq:boundmax2f},
  the second inequality comes from \eqref{eq:boundh} and $\lambda \in (0, 1]$,
  and the last inequality follows from
  \begin{align*}
    \frac{(1 + \lambda)^2}{\lambda}
    &
    =
    \lambda
    +
    2
    +
    \frac{1}{\lambda}
    \leq
    3
    +
    \frac{1}{\lambda}
    =
    4
    +
    \left(
    \frac{1}{\lambda}
    -
    1
    \right) ,
    \\
    \log \left(
      1 +
      \frac{\sigma_i^2}{\lambda \Delta_i}
    \right)
    &
    \leq
    \frac{1}{\lambda}
    \log \left(
      1 +
      \frac{\sigma_i^2}{\Delta_i}
    \right)
    \leq
    \log \left(
      1 +
      \frac{\sigma_i^2}{\Delta_i}
    \right)
    +
    \left(
    \frac{1}{\lambda}
    - 1
    \right)
    \frac{\sigma_i^2}{\Delta_i} .
  \end{align*}
  Combining \eqref{eq:limRTgamma-corrupted} and \eqref{eq:boundlamfDx},
  we obtain
  \begin{align}
    \frac{R_{i^*}(T)}{\log T}
    &
    \leq
    \sum_{i \in [K] \setminus \{i^*\}}
    \max\left\{
      4
    \frac{\sigma_i^2}{ \Delta_i}
    +
    c
    \log \left(
      1 +
      \frac{\sigma_i^2}{\Delta_i}
    \right),
    2
    \bet
    \right\}
    +
    2 \bet
    \nonumber
    \\
    &
    \quad
    +
    2
    \left(
      2 + \frac{\delta}{\bet}
    \right)
    \sqrt{
      \frac{K-1}{\gamma} C
    }
    +
    2 \lambda \frac{C}{\gamma}
    +
    (1+c)
    \left(
      \frac{1}{\lambda} - 1
    \right)
    \sum_{i \in [K] \setminus \{ i^* \}}
    \frac{\sigma_i^2}{\Delta_i}
    +
    O
    \left( \frac{1}{\gamma^{1/4}} \right) .
    \label{eq:boundRTlogTcorrupted}
  \end{align}
  If we choose
  \begin{align}
    \lambda
    =
    \sqrt{
      \frac{
        \gamma
        \sum_{i \in [K] \setminus \{ i^* \}}
        \frac{\sigma_i^2}{\Delta_i}
      }{
        \gamma
        \sum_{i \in [K] \setminus \{ i^* \}}
        \frac{\sigma_i^2}{\Delta_i}
        +
        2 C
      }
    },
  \end{align}
  we have
  \begin{align*}
    \lambda
    &
    \leq
    \sqrt{
      \frac{
        \gamma
        \sum_{i \in [K] \setminus \{ i^* \}}
        \frac{\sigma_i^2}{\Delta_i}
      }{
        2 C
      }
    },
    \\
    \frac{1}{\lambda} - 1
    &
    =
     \sqrt{
      1 +
      \frac{ 2 C }{
        \gamma
        \sum_{i \in [K] \setminus \{ i^* \}}
        \frac{\sigma_i^2}{\Delta_i}
      }
     }
     - 1
    \leq
    \sqrt{
      \frac{
        2C
      }{
        \gamma
        \sum_{i \in [K] \setminus \{ i^* \}}
        \frac{\sigma_i^2}{\Delta_i}
      }
    },
  \end{align*}
  which implies that
  \begin{align*}
    &
    2
    \left(
      2 + \frac{\delta}{\bet}
    \right)
    \sqrt{
      \frac{K-1}{\gamma} C
    }
    +
    2 \lambda \frac{C}{\gamma}
    +
    (1+c)
    \left(
      \frac{1}{\lambda} - 1
    \right)
    \sum_{i \in [K] \setminus \{ i^* \}}
    \frac{\sigma_i^2}{\Delta_i}
    \\
    &
    =
    O\left(
      \frac{C}{\gamma}
      \sum_{i \in [K] \setminus \{ i^* \}}
      \left(
      \frac{\sigma_i^2}{\Delta_i}
      +
      1
      \right)
    \right).
  \end{align*}
  From this and \eqref{eq:boundRTlogTcorrupted},
  recalling that $\gamma = \log T$ and $\bet = 1+\epsilon$,
  we obtain
  \begin{align}
    R_{i^*} (T)
    &
    \leq
    \sum_{i \in [K] \setminus \{i^*\}}
    \max\left\{
      4
    \frac{\sigma_i^2}{ \Delta_i}
    +
    c
    \log \left(
      1 +
      \frac{\sigma_i^2}{\Delta_i}
    \right),
    2 (1+ \epsilon)
    \right\}
    \log T
    +
    2 (1 + \epsilon ) \log T
    \nonumber
    \\
    &
    \quad
    +
    O \left(
      C
      \sum_{i \in [K] \setminus \{ i^* \}}
      \left(
      \frac{\sigma_i^2}{\Delta_i}
      +
      1
      \right)
      \log T
    \right)
    +
    o ( \log T ) ,
    \label{eq:thm1-bound-corrupted}
  \end{align}
  which completes the proof of Proposition~\ref{prop:bound-corrupted}.
\end{proof}

\subsection{Proof of the regret bound of (\ref{eq:thm1-bound-adv}) in adversarial settings}
\label{sec:appendix-adv}
In this subsection,
we show the regret bounds described in \eqref{eq:thm1-bound-adv} that holds for adversarial settings.
In the following,
we show
$R(T) \leq 2 \sqrt{ K Q_\infty \log T } + O(K \log T)$.
From Lemma~\ref{lem:boundRTbeta} combined with~\eqref{eq:defalphabeta} and Lemma~\ref{lem:bound-sumalpha-suboptimal},
for any $m^* \in [0, 1]^K$,
we have
\begin{align}
  R(T)
  &
  \leq
  \gamma \sum_{i=1}^K
  \E \left[
    2 \beta_i(T+1)
    -
    \beta_i(T)
    +
    2 \delta \log \frac{\beta_i(T+1)}{\beta_i(1)}
  \right]
  +
  2 K (1 + \delta)
  \nonumber
  \\
  &
  \leq
  2 \gamma \sum_{i=1}^K
  \E \left[
    \beta_i(T+1)
  \right]
  +
  O( K \gamma )
  \nonumber
  \\
  &
  =
  2 \gamma \sum_{i=1}^K
  \E \left[
    \sqrt{
      \bet^2
      +
      \frac{1}{\gamma}
      \sum_{t=1}^T \alpha_i(t)
    }
  \right]
  +
  O( K \gamma )
  \nonumber
  \\
  &
  \leq
  2 \gamma \sum_{i=1}^K
  \E \left[
    \sqrt{
      \bet^2
      +
      \frac{1}{\gamma}
      \left(
      \sum_{t=1}^T
      \mathbf{1}[I(t) = i] (\ell_i(t) - m^*_i)^2
      +
      \log (1 + N_i(T))
      +
      \frac{5}{4}
      \right)
    }
  \right]
  +
  O( K \gamma )
  \nonumber
  \\
  &
  \leq
  2 \gamma \sum_{i=1}^K
  \E \left[
    \sqrt{
      \frac{1}{\gamma}
      \sum_{t=1}^T
      \mathbf{1}[I(t) = i] (\ell_i(t) - m^*_i)^2
    }
  \right]
  +
  O( K \gamma )
  \nonumber
  \\
  &
  \leq
  2 \E \left[
    \sqrt{
      K \gamma  \sum_{i=1}^K
      \sum_{t=1}^T
      \mathbf{1}[I(t) = i] (\ell_i(t) - m^*_i)^2
    }
  \right]
  +
  O( K \gamma )
  \nonumber
  \\
  &
  =
  2 \E \left[
    \sqrt{
      K \gamma
      \sum_{t=1}^T
      (\ell_{I(t)}(t) - m^*_{I(t)})^2
    }
  \right]
  +
  O( K \gamma ),
  \label{eq:boundRTm*}
\end{align}
where the first and the third inequalities follow from Lemmas~\ref{lem:boundRTbeta} and \ref{lem:bound-sumalpha-suboptimal},
respectively,
the second inequality comes from $\beta_i(T+1) = O(T)$ following from \eqref{eq:defalphabeta},
and the last inequality is due to the Cauchy-Schwarz inequality.
As
we have
$
  (\ell_{I(t)}(t) - m^*_{I(t)})^2
  \leq \| \ell(t) - m^* \|_{\infty}^2
$,
\eqref{eq:boundRTm*} implies
\begin{align*}
  R(T)
  \leq
  2 \E \left[
    \sqrt{
      K \gamma
      \sum_{t=1}^T
      \|\ell(t) - m^* \|_{\infty}^2
    }
  \right]
  +
  O( K \gamma ).
\end{align*}
As this holds for any $m^* \in [0, 1]^K$,
we have
$R(T) \leq 2 \sqrt{ K Q_{\infty} \log T } + O(K \log T)$.

In the following,
we show
$R(T) \leq 2 \sqrt{ K L^* \log T } + O(K \log T)$
and
$R(T) \leq 2 \sqrt{ K (T-L^*) \log T } + O(K \log T)$.
From
\eqref{eq:boundRTm*} with $m^* = 0$,
for $i^* \in \argmin_{i} \E \left[ \sum_{t=1}^T \ell_i(t) \right] $,
we have
\begin{align}
  R(T)
  &
  \leq
  2 \E \left[
    \sqrt{
      K \gamma
      \sum_{t=1}^T
      \ell_{I(t)}(t)^2
    }
  \right]
  +
  O( K \gamma )
  \nonumber
  \\
  &
  \leq
  2 \E \left[
    \sqrt{
      K \gamma
      \sum_{t=1}^T
      \ell_{I(t)}(t)
    }
  \right]
  +
  O( K \gamma )
  \nonumber
  \\
  &
  \leq
  2 \E \left[
    \sqrt{
      K \gamma
      \left(
      \sum_{t=1}^T
      \left(
      \ell_{I(t)}(t)
      -
      \ell_{i^*}(t)
      \right)
      +
      \sum_{t=1}^T
      \ell_{i^*}(t)
      \right)
    }
  \right]
  +
  O( K \gamma )
  \nonumber
  \\
  &
  \leq
  2
  \sqrt{
    K \gamma
    \left(
    \E \left[
    \sum_{t=1}^T
    \left(
    \ell_{I(t)}(t)
    -
    \ell_{i^*}(t)
    \right)
    \right]
    +
    \E \left[
    \sum_{t=1}^T
    \ell_{i^*}(t)
    \right]
    \right)
  }
  +
  O( K \gamma )
  \nonumber
  \\
  &
  =
  2
  \sqrt{
    K \gamma
    \left(
    R(T)
    +
    L^*
    \right)
  }
  +
  O( K \gamma ),
\end{align}
where the second inequality follows from $\ell_i(t) \in [0, 1]$
and the forth inequaliy is due to the Cauchy-Schwarz inequality.
By solving this inequation in $R(T)$,
we obtain
\begin{align}
  \nonumber
  R(T) \leq 2 \sqrt{K \gamma L^*} + O (K \gamma)
  =
  2 \sqrt{K L^* \log T}
  +
  O(K\log T).
\end{align}
Similarly,
from \eqref{eq:boundRTm*} with $m^* = \mathbf{1}$,
we have
\begin{align}
  R(T)
  &
  \leq
  2 \E \left[
    \sqrt{
      K \gamma
      \sum_{t=1}^T
      (\ell_{I(t)}(t) - 1)^2
    }
  \right]
  +
  O( K \gamma )
  \nonumber
  \\
  &
  \leq
  2 \E \left[
    \sqrt{
      K \gamma
      \sum_{t=1}^T
      (1 - \ell_{I(t)}(t) )
    }
  \right]
  +
  O( K \gamma )
  \nonumber
  \\
  &
  \leq
  2 \E \left[
    \sqrt{
      K \gamma
      \left(
      T
      -
      \sum_{t=1}^T
      \ell_{i^*}(t)
      -
      \sum_{t=1}^T
      \left(
      \ell_{I(t)}(t)
      -
      \ell_{i^*}(t)
      \right)
      \right)
    }
  \right]
  +
  O( K \gamma )
  \nonumber
  \\
  &
  \leq
  2
  \sqrt{
    K \gamma
    \left(
      T - L^* - R(T)
    \right)
  }
  +
  O( K \gamma ),
\end{align}
which implies
\begin{align}
  \nonumber
  R(T) \leq 2 \sqrt{K \gamma (T - L^*)} + O (K \gamma)
  =
  2 \sqrt{K (T - L^*) \log T}
  +
  O(K\log T).
\end{align}

\section{Proof of Proposition \ref{prop_moment}}
\label{append_lower}
The achievability result is a rephrasing of \citet[Theorem 1]{honda_moment}.
It is shown in \citet[Theorem 4]{honda_moment} that
there exist distributions such that any consistent algorithm satisfies
\begin{align}
\liminf_{T\to\infty}\frac{R(T)}{\log T}
\ge
\sum_{i\neq i^*}\frac{\Delta_i}{D_{\inf}^{(2)}(M_{i,1},M_{i,2},1-\mu^*)},\label{bound_dinf}
\end{align}
where $M_{i,1}=\E[1-\ell_i]$ and $M_{i,2}=\E[(1-\ell_i)^2]$ are the first and second moments of the reward, respectively,
and
\begin{align}
D_{\inf}^{(2)}(M_1,M_2,1-\mu^*)
&=
\frac{(1-M_1)^2}{1-2M_1+M_2}\log \left(1-\left(\frac{M_1-M_2}{1-M_1}-(1-\mu^*)\right)\nu^{(2)}\right)\nn
&\quad\;\;\;\;+\frac{M_2-M_1^2}{1-2M_1+M_2}\log \left(1-\mu^*\nu^{(2)}\right)\label{dinf_expression}
\end{align}
for
\begin{align}
\nu^{(2)}=\frac{(1-M_1)(M_1-1+\mu^*)}{(1-M_1)(1-\mu^*)^2-(1-M_2)(1-\mu^*)+M_1-M_2}.\label{nu_expression}
\end{align}
Here note that $\mu^*$ in \citet{honda_moment} is replaced with $1-\mu^*$ since we are considering the loss instead of the reward,

By substituting
$M_{1}=1-\mu_i$ and $M_{2}=(1-\mu_i)^2+\sigma_i^2$, we have
\begin{align}
\nu^{(2)}
&=\frac{\mu_i(\mu^*-\mu_i)}{\mu_i(1-\mu^*)^2-(2\mu_i - \mu_i^2 -\sigma_i^2)(1-\mu^*)+\mu_i-\mu_i^2-\sigma_i^2}\nn
&=\frac{\mu_i(\mu^*-\mu_i)}{\mu_i-2\mu_i\mu^*+\mu_i(\mu^*)^2-2\mu_i + \mu_i^2 +\sigma_i^2
+2\mu_i\mu^* - \mu_i^2\mu^* -\sigma_i^2\mu^*
+\mu_i-\mu_i^2-\sigma_i^2}\nn
&=\frac{\mu_i(\mu^*-\mu_i)}{\mu_i(\mu^*)^2
 - \mu_i^2\mu^* -\sigma_i^2\mu^*
}\nn
&=
\frac{\mu_i(\mu^*-\mu_i)}{\mu^*(\mu_i\mu^*
 - \mu_i^2 -\sigma_i^2)
}.
\end{align}
Therefore,
\begin{align}
\lefteqn{
D_{\inf}^{(2)}(M_1,M_2,1-\mu^*)
}\nn
&=
\frac{\mu_i^2}{\sigma_i^2+\mu_i^2}\log \left(1-\left(\frac{\mu_i-\mu_i^2-\sigma_i^2}{\mu_i}-(1-\mu^*)\right)\nu^{(2)}\right)
+\frac{\sigma_i^2}{\sigma_i^2+\mu_i^2}\log \left(1-\mu^*\nu^{(2)}\right),\nn
&=
\frac{\mu_i^2}{\sigma_i^2+\mu_i^2}\log \left(1-\left(\frac{\mu_i\mu^*-\mu_i^2-\sigma_i^2}{\mu_i}\right)\nu^{(2)}\right)
+\frac{\sigma_i^2}{\sigma_i^2+\mu_i^2}\log \left(1-\mu^*\nu^{(2)}\right)\nn
&=
\frac{\mu_i^2}{\sigma_i^2+\mu_i^2}\log \left(1-\frac{\mu^*-\mu_i}{\mu^*}\right)
+\frac{\sigma_i^2}{\sigma_i^2+\mu_i^2}\log \left(1-\frac{\mu_i(\mu^*-\mu_i)}{\mu_i\mu^*
 - \mu_i^2 -\sigma_i^2}\right)\nn
&=
\frac{\mu_i^2}{\sigma_i^2+\mu_i^2}\log \frac{\mu_i}{\mu^*}
+\frac{\sigma_i^2}{\sigma_i^2+\mu_i^2}\log \frac{\sigma_i^2}{\mu_i^2 +\sigma_i^2-\mu_i\mu^*}\nn
&=
\frac{\mu_i^2}{\sigma_i^2+\mu_i^2}\log \frac{\mu_i}{\mu^*}
+\frac{\sigma_i^2}{\sigma_i^2+\mu_i^2}\log \frac{\sigma_i^2}{\sigma_i^2+\mu_i\Delta_i},
\end{align}
which leads to the desired result combined with \eqref{bound_dinf}.
\qed

\section{Modification for Path-Length Regret Bounds}
\label{sec:path-length}
By modifying the update rule for $\{ m(t) \}$ in the proposed algorithm,
we can obtain a path-length regret bound
dependent on $V_1 = \E[ \sum_{t=1}^{T-1} \| \ell(t) - \ell(t-1) \|_1 ]$.
Let us consider computing $m(t)$ in a similar way to that of \citet{ito2021parameter},
i.e.,
set
\begin{align}
  m_i(1) = \frac{1}{2}
  \quad
  ( i \in [K]),
  \quad
  m_i(t+1)
  =
  \left\{
    \begin{array}{ll}
      (1- \eta) m_i(t) + \eta \ell_i(t) & \mbox{if} \quad i = I(t)
      \\
      m_i(t) & \mbox{if} \quad i \neq I(t)
    \end{array}
  \right.
\end{align}
with a parameter $\eta \in (0, 1/2)$,
instead of \eqref{eq:defmi} in this paper.
Then,
from Proposition 13 in \citet{ito2021parameter},
we have
\begin{align}
  \sum_{t=1}^T \sum_{i=1}^K \alpha_i(t)
  \leq
  \frac{1}{1-2 \eta}
  \sum_{t=1}^T
  (\ell_{I(t)}(t) - m^*_{I(t)}(t))^2
  +
  \frac{2}{\eta(1- 2 \eta)}
  \left(
    \frac{K}{8}
    +
    \sum_{t=1}^{T-1}
    \|
    m^*(t)
    -
    m^*(t+1)
    \|_1
  \right)
  \label{eq:boundalphaeta}
\end{align}
for any sequence $\{ m^*(t) \}_{t=1}^T \subseteq [0, 1]^K$.
From \eqref{eq:boundalphaeta} with
$m^*(t) = \ell(t)$ for all $t \in [T]$,
we have
\begin{align*}
  \E \left[
    \sum_{t=1}^T \sum_{i=1}^K \alpha_i(t)
  \right]
  &
  \leq
  \frac{2}{\eta(1- 2 \eta)}
  \left(
    \frac{K}{8}
    +
    \E \left[
    \sum_{t=1}^{T-1}
    \|
    \ell(t)
    -
    \ell(t+1)
    \|_1
    \right]
  \right)
  =
  \frac{2}{\eta(1- 2 \eta)}
  \left(
    \frac{K}{8}
    +
  V_1
  \right).
\end{align*}
From this,
via the argument in Section~\ref{sec:appendix-adv},
we have the following path-length regret bound:
\begin{align}
  \label{eq:path-length-appendix}
  R(T)
  \leq
  \sqrt{
    \frac{K}{\eta(1-2\eta)}
    \left(
    8 V_1
    +
    K 
    \right)
    \log T
  }
  +
  O(K \log T).
\end{align}

Further,
the modified version preserves the regret bounds achieved by the original version
though the leading constant factors get increased.
By considering \eqref{eq:boundalphaeta} with $m^*(t) = m^*$ for all $t$,
we obtain
\begin{align*}
  \E \left[
    \sum_{t=1}^T \sum_{i=1}^K \alpha_i(t)
  \right]
  &
  \leq
  \frac{1}{1 - 2 \eta}
  \min_{ m^* \in [0, 1]^K }
  \left\{
    \E \left[
    \sum_{t=1}^T
    ( \ell_{I(t)}(t) - m^*_{I(t)})^2
    \right]
  \right\}
  +
  \frac{K}{4 \eta(1- 2 \eta)}
  \\
  &
  \leq
  \frac{1}{1 - 2 \eta}
  \min
  \left\{
    Q_{\infty},
    R(T) + L^*,
    T - L^* - R(T)
  \right\}
  +
  \frac{K}{4 \eta(1- 2 \eta)},
\end{align*}
which implies
\begin{align}
  \label{eq:data-dependent-appendix}
  R(T)
  \leq
  \sqrt{
    \frac{K}{1-2\eta}
    \left(
    \min\{
      T,
      4 Q_{\infty},
      4 L^*,
      4 (T - L^*)
    \}
    +
    \frac{K}{\eta}
    \right)
    \log T
  }
  +
  O(K \log T) .
\end{align}
Combining \eqref{eq:path-length-appendix} and \eqref{eq:data-dependent-appendix},
we obtain \eqref{eq:bound-pathlength}.

Similarly,
in stochastic settings,
we have
\begin{align*}
  \E \left[
  \sum_{t=1}^{T}
  \alpha_i(t)
  \right]
  &
  \leq
  \frac{1}{1- 2 \eta}
  \E
  \left[
    \sum_{t=1}^T
    \mathbf{1}[I(t)=i]
    (\ell_i(t) - \mu_i)^2
  \right]
  +
  \frac{1}{4 \eta (1 - 2 \eta)}
  \\
  &
  =
  \frac{1}{1- 2 \eta}
  \sigma_i^2 P_i
  +
  \frac{1}{4 \eta (1 - 2 \eta)}.
\end{align*}
By using this instead of \eqref{eq:bound-sumalpha-suboptimal},
via the argument in Section~\ref{sec:appendix-sto},
we can obtain
\begin{align*}
  R(T)
  &
  \leq
  \frac{1}{1- 2 \eta}
    \left(
    \sum_{i \in [K] \setminus \{ i^* \}}\max\left\{ 4 \frac{\sigma_i^2}{\Delta_i}
    +
    c
    \log \left(
      1 + \frac{\sigma_i^2}{\Delta_i}
    \right),\,2(1+\epsilon)\right\}
    +
    2
    (1+\epsilon)
    \right)
    \log T
    \\
    &
    \quad
    +
    O \left(
      K
      \sqrt{\frac{\log T}{ \eta(1-2\eta)}}
    \right)
    +
    o (\log T)
\end{align*}
in stochastic settings.

\end{document}